\documentclass[11pt,letter]{article}
\usepackage{amsmath,amsthm,amssymb,amsbsy,amsfonts,amssymb}
\usepackage{graphicx,color,tikz,xcolor}
\usepackage{enumerate}
\usepackage{bbm,bbold,bm,dsfont}
\usepackage{natbib}
\usepackage{multirow}
\usepackage{subcaption}
\usepackage{enumitem}
\usepackage{booktabs}
\usepackage{algorithm, algorithmicx,algpseudocode}
\usepackage{hyperref} 
\usepackage{tikz}
\usetikzlibrary{arrows}
\usetikzlibrary{positioning}
\usetikzlibrary{calc}
\newdimen\nodeDist
\usepackage{url} 

\newcommand{\blind}{1}

\usepackage[doublespacing]{setspace}

\hypersetup{colorlinks=true,       
  linkcolor=blue,
  citecolor=blue,        
  filecolor=magenta,      
  urlcolor=cyan           
}

\usetikzlibrary{decorations.shapes}
\tikzset{decorate sep/.style 2 args=
{decorate,decoration={shape backgrounds,shape=circle,shape size=#1,shape sep=#2}}}

\newcommand{\inv}{^{\raisebox{.2ex}{$\scriptscriptstyle-1$}}}

\newcommand{\minus}{\scalebox{0.45}[1.0]{$-$}}

\algnewcommand{\minus}{\scalebox{0.45}[1.0]{$-$}}

\DeclareMathOperator{\iid}{\stackrel{\mbox{\tiny iid} }{\sim}}

\newcommand{\X}{\mathbf{X}}
\newcommand{\w}{\mathrm{w}}

\newcommand{\cc}{\mathbf{c}}
\newcommand{\y}{\mathrm{y}}

\newcommand{\N}{\mbox{{\small\textsc{N}}}}

\newcommand{\E}{\mbox{E}}

\newcommand{\res}{\mathrm{r}}

\newcommand{\x}{\mathrm{x}}

\newtheorem{theorem}{Theorem}

\newtheorem{proposition}{Proposition}
\newtheorem*{lemma*}{Lemma}
\newtheorem{lemma}{Lemma}
\newtheorem{condition}{Condition}

\newenvironment{remark}[1][Remark]{\begin{trivlist}
\item[\hskip \labelsep {\bfseries #1}]}{\end{trivlist}}

\newcommand\blank{{}\cdot {}}

\let\originalleft\left
\let\originalright\right
\def\left#1{\mathopen{}\originalleft#1}
\def\right#1{\originalright#1\mathclose{}}

\definecolor{block}{RGB}{0,162,232}

\def\blockaux#1(#2,#3)#4(#5,#6){%
  \draw[fill={#1}, draw=white]
  let \p1=(#2,#3),
      \p2=(#5,#6),
      \p3=(#2+#5,#3+#6),
      \p4=(#2+#5/2,#3+#6/2)
  in
    (\p1) rectangle (\p3)
    (\p4) node {$#4$}
  ;%
}

\begin{document}

\def\spacingset#1{\renewcommand{\baselinestretch}%
	{#1}\small\normalsize} \spacingset{1}


\if1\blind
	{
		\title{\bf Stochastic tree ensembles \\ for regularized nonlinear regression}
		 \author{Jingyu He\\
			City University of Hong Kong\\
			and \\
			P. Richard Hahn\footnote{Thanks to Nick Polson for helpful comments. This work was partially supported by a Facebook Research Award. }\\
			Arizona State University}
		\maketitle
	} \fi

\if0\blind
{
	\title{\bf Stochastic tree ensembles \\ for regularized nonlinear regression}
	\maketitle
} \fi

\bigskip
\begin{abstract}
	This paper develops a novel stochastic tree ensemble method for nonlinear regression, referred to as Accelerated Bayesian Additive Regression Trees, or XBART. By combining regularization and stochastic search strategies from Bayesian modeling with computationally efficient techniques from recursive partitioning algorithms, XBART attains state-of-the-art performance at prediction and function estimation. Simulation studies demonstrate that XBART provides accurate point-wise estimates of the mean function and does so faster than popular alternatives, such as BART, XGBoost, and neural networks (using Keras) on a variety of test functions. Additionally, it is demonstrated that using XBART to initialize the standard BART MCMC algorithm considerably improves credible interval coverage and reduces total run-time. Finally, two basic theoretical results are established: the single tree version of the model is asymptotically consistent and the Markov chain produced by the ensemble version of the algorithm has a unique stationary distribution. 

\end{abstract}

\noindent%
\vfill

\newpage
\spacingset{1.5} 

\newpage

\section{Introduction}
Tree-based algorithms for supervised learning, such as Classification and Regression Trees (CART) \citep{breiman1984classification}, random forests \citep{breiman2001random} and boosted regression trees \citep{friedman2002stochastic, chen2016xgboost}, are popular due to their speed and accuracy in out-of-sample prediction tasks. For instance, in section 10.7 of \cite{hastie2005elements}, an influential textbook, we read
\small
\begin{quote}
	Of all the well-known learning methods, decision trees come closest to meeting the requirements for serving as an off-the-shelf procedure for data mining. They are relatively fast to construct and they produce interpretable models (if the trees are small). They naturally incorporate mixtures of numeric and categorical predictor variables and missing values. They are invariant under (strictly monotone) transformations of the individual predictors. As a result, scaling and/or more general transformations are not an issue, and they are immune to the effects of predictor outliers. They perform internal feature selection as an integral part of the procedure. They are thereby resistant, if not completely immune, to the inclusion of many irrelevant predictor variables. These properties of decision trees are largely the reason that they have emerged as the most popular learning method for data mining. Trees have one aspect that prevents them from being the ideal tool for predictive learning, namely inaccuracy. They seldom provide predictive accuracy comparable to the best that can be achieved with the data at hand. Boosting decision trees improves their accuracy, often dramatically. 
\end{quote}
\normalsize
In 2016, the XGBoost algorithm was introduced \citep{chen2016xgboost} and has quickly become a go-to tool for data scientists working in industry:
\small
\begin{quote}
	Take the challenges hosted by the machine learning competition site Kaggle for example. Among the 29 challenge winning solutions 3 published at Kaggle’s blog during 2015, 17 solutions used XGBoost. Among these solutions, eight solely used XGBoost to train the model, while most others combined XGBoost with neural nets in ensembles. For comparison, the second most popular method, deep neural nets, was used in 11 solutions. The success of the system was also witnessed in KDDCup 2015, where XGBoost was used by every winning team in the top 10 \citep{bekkerman2015}.
\end{quote}
\normalsize
\cite{pafka2015} performed simulation comparisons of the performance of XGBoost and other gradient boosting methods and random forests, and conclude that XGBoost is fast, memory efficient and of high accuracy. \cite{brownlee2016} summarizes insightful quotes and praises from Kaggle competition winners on XGBoost. 

At the same time, an older tree-based method, random forests \citep{breiman2001random}, is often faster and competitively accurate, especially on low-signal data sets. In Section 15.2 of Elements of Statistical Learning \cite{hastie2005elements}, the authors remark:
\begin{quote}
In our experience random forests do remarkably well, with very little tuning required.
\end{quote}
In short, it would not be much of an exaggeration to say that XGBoost and random forests are the standard bearers for classification and regression problems with the unstructured (tabular) mixed numeric-categorical data that are common in many industries. 

Still, there may be room for improvement, in at least two respects. One, it would be preferable to have a single method that could replace the two, a method that works well on both high and low signal data sets. To some extent XGBoost can avoid overfitting, and hence achieve model fits more like random forests, by choosing a less aggressive learning rate (one of the algorithms tunable parameters), but selecting this parameter by cross-validation can be costly, essentially undoing the tremendous speed advantage for which XGBoost is famous. Two, in some situations it would be useful to have a prediction interval in addition to a point estimate. Although some approaches, such as conformal prediction \citep{lei2018distribution}, can be used to augment XGBoost or random forests, doing so comes at a steep computational cost. In these two respects --- regularization that is adaptive to problem difficulty (data quality) and availability of an associated uncertainty measure --- a third popular tree ensemble method arises as a competitor:  Bayesian additive regression trees (BART) \citep{chipman2010bart} is a model-based method that is comparatively robust to the choice of tuning-parameters and, as a Bayesian method, provides posterior uncertainty quantification. Due to these strengths, BART has inspired a considerable body of research in recent years; for a comprehensive review of this literature, see \cite{linero2017review} and \cite{hill2020bayesian}. 

However, relative to random forests and XGBoost, BART models take much longer to fit because the underlying random walk Metropolis-Hastings Markov chain Monte Carlo (MCMC) algorithm can be slow to converge. The present paper develops a novel stochastic tree ensemble method, referred to as Accelerated Bayesian Additive Regression Trees, or XBART, that greatly reduces the time needed to fit BART models. XBART takes its regularization and parameter sampling steps from BART, while retaining the recursive tree-growing process from traditional tree-based methods. The result is a stochastic tree sampling algorithm that is substantially faster than BART while retaining its state-of-the-art predictive accuracy. A general form of the recursive stochastic algorithm is presented in Section \ref{sec:XBART}; Section \ref{sec:gaussian} tailors the algorithm to a regression model with additive Gaussian errors. Extensive simulation studies and empirical data examples demonstrate the efficacy of the new approach in Section \ref{sec:example}.

Furthermore, XBART works not only as a stand-alone machine learning algorithm but can be used to initialize a BART MCMC sampler (warm-start BART, Section \ref{sec:warm_start}), resulting in faster fully Bayesian inference with improved posterior exploration as indicated by posterior credible intervals (for the mean function) with better coverage (for a fixed number of posterior samples).

Finally, basic theoretical properties of the new algorithm are investigated in Section \ref{sec:theory}. First, recent consistency results concerning CART and random forests \citep{scornet2015consistency} are shown to apply to XBART with a single tree. Second, it is shown that the XBART forest algorithm defines a finite-space Markov chain with stationary distribution.

The XBART algorithm was first presented in \cite{he2019xbart}; relative to this initial description, the version of XBART presented here is expressed in more generality, the simulation studies are more extensive, and the real data examples, the warm-start strategy and the theoretical results are entirely new.


\section{A recursive, stochastic fitting algorithm}\label{sec:XBART}

The goal of supervised learning is to predict a scalar random variable $Y \in \mathcal{Y}$ by a length $p$ covariate vector $\x = (x_1, \cdots, x_p)^t \in \mathcal{X}$. This section presents a stochastic, recursive algorithm for supervised learning with decision or regression trees. The algorithm is first developed for fitting a single tree and then extended to tree ensembles, or forests. This section describes the algorithm in terms that apply to a general likelihood function that can be used for classification or regression; the remainder of the paper focuses on regression with a continuous univariate response variable.

\subsection{Fitting a single tree recursively and stochastically}\label{sec:gfr}
A tree $T$ is a set of split rules defining a rectangular partition of the covariate space to $\{\mathcal{A}_{1}, \cdots, \mathcal{A}_{B}\}$, where $B$ is the total number of terminal nodes of tree $T$. The split rule is a pair of $(x_i, c)$, indicating variable to split and the value at which it cuts. Each rectangular cell $\mathcal{A}_{b}$ is associated with leaf parameter $\mu_{b}$ and the pair $(T, \boldsymbol{\mu})$ parameterizes a step function $g(\cdot)$ on covariate space:
\begin{equation*}
	g\left(\x ; T, \boldsymbol{\mu}\right) = \mu_{b}, \quad \text{if } \x \in \mathcal{A}_{b}
\end{equation*}
 where $\boldsymbol{\mu} = (\mu_{1}, \cdots, \mu_{B})^t$ denotes a vector of all leaf parameters. 
Figure \ref{fig:tree_example} depicts a simple regression tree using two variables $\x = (x_1, x_2)^t \in [0,1]\times[0,1]$. The left panel shows a split rule structure and the right panel plots the corresponding partition of the space and the associated leaf parameters. The response distribution at a point $\x$, upon which predictions are based,  is determined by the leaf parameter, $\mu_b$, associated with the cell containing $\x$. 
\begin{figure}[!htbp]
	\begin{subfigure}{0.5\textwidth}
		\begin{center}
			\begin{tikzpicture}[
					scale=0.8,
					node/.style={%
							draw,
							rectangle,
						},
					node2/.style={%
							draw,
							circle,
						},
				]

				\node [node] (A) {$x_1\leq0.8$};
				\path (A) ++(-135:\nodeDist) node [node2] (B) {$\mu_{1}$};
				\path (A) ++(-45:\nodeDist) node [node] (C) {$x_2\leq0.4$};
				\path (C) ++(-135:\nodeDist) node [node2] (D) {$\mu_{2}$};
				\path (C) ++(-45:\nodeDist) node [node2] (E) {$\mu_{3}$};

				\draw (A) -- (B) node [left,pos=0.25] {no}(A);
				\draw (A) -- (C) node [right,pos=0.25] {yes}(A);
				\draw (C) -- (D) node [left,pos=0.25] {no}(A);
				\draw (C) -- (E) node [right,pos=0.25] {yes}(A);
			\end{tikzpicture}
		\end{center}

	\end{subfigure}
	\begin{subfigure}{0.5\textwidth}

		\begin{center}
			\begin{tikzpicture}[scale=3]
				\draw [thick, -] (0,1) -- (0,0) -- (1,0) -- (1,1)--(0,1);
				\draw [thin, -] (0.8, 1) -- (0.8, 0);
				\draw [thin, -] (0.0, 0.4) -- (0.8, 0.4);
				\node at (-0.1,0.4) {0.4};
				\node at (0.8,-0.1) {0.8};
				\node at (-0.1, -0.1) {0};
				\node at (1, -0.1) {1};
				\node at (-0.1,1) {1};
				\node at (0.5,-0.2) {$x_1$};
				\node at (-0.3,0.5) {$x_2$};
				\node at (0.9,0.5) {$\mu_{1}$};
				\node at (0.4,0.7) {$\mu_{2}$};
				\node at (0.4,0.2) {$\mu_{3}$};
			\end{tikzpicture}
		\end{center}
	\end{subfigure}
	\caption{A regression tree on two variables depicting two split rules (cutpoints) and three leaf nodes.}\label{fig:tree_example}
\end{figure}
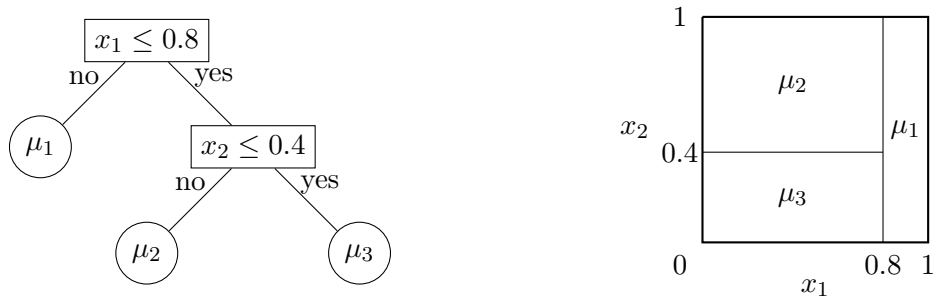

Regression trees may be ``grown'' by recursive partitioning:  the best split rule (cutpoint) is determined by examining all cutpoint candidates along each variable, where best is defined by a specified split criterion. The data set is then divided according to this split rule (cutpoint) and the process is repeated on the resulting disjoint subsets until a stopping condition is met. Algorithm \ref{alg:recursion} depicts pseudocode for recursive partitioning.

\begin{algorithm}
\small
\caption{Recursive partitioning}\label{alg:recursion}
	\begin{algorithmic}[1]
	\State A length-$n$ response vector $\y$ and a $n \times p$ predictor matrix $\X$.
		\Procedure{RecursivePartition}{$\mathcal{I}$} \Comment{Row indices $\mathcal{I}$, a subset of $1, \dots, n$.}
			\If{Stop($\mathcal{I}$) = FALSE} 
				\State Compute split criterion based on partitions of $\y$ corresponding to each candidate cutpoint defined by $\mathcal{I}$ and $\X$. 
				\State Define disjoint subsets $\mathcal{I}_{\mbox{left}}$ and $\mathcal{I}_{\mbox{right}}$ based on the optimal cutpoint.
				\State \Call{RecursivePartition}{$\mathcal{I}_{\mbox{left}}$}
				\State \Call{RecursivePartition}{$\mathcal{I}_{\mbox{right}}$}	
		\Else
			
		\Return a single prediction based on observations $y_i$ for $i \in \mathcal{I}$.
		\EndIf
		\EndProcedure
	\end{algorithmic}
\end{algorithm}

Most widely-used regression tree methods are some version of Algorithm \ref{alg:recursion}, differing from one another in terms of the split criteria and stopping conditions they employ. Split criteria are functions of a split rule (a cutpoint), measuring homogeneity within the two child nodes produced by the implied split; for instance, CART \citep{breiman1984classification} uses a squared error  split criteria.  Frequently used stop conditions include maximum depth of a tree, minimal number of data observations within a leaf node, or a threshold for percent change of split criteria from parent to child nodes. 

%

The XBART algorithm is a modification of Algorithm \ref{alg:recursion} in which the partition cutpoints and also the stopping condition are determined \textit{stochastically}. The splitting criterion will be motivated by analogy with an integrated likelihood calculation arising from the BART MCMC algorithm \citep{chipman2010bart}. Here we present the result in a highly generic form; details specific to the Gaussian mean regression model are given in Section \ref{sec:gaussian}.

\subsubsection{The XBART marginal likelihood split criterion}
To establish notation: The predictor matrix $\X$, with dimension $n\times p$, consists of $n$ observations and $p$ variables. Tree regression methods partition the data according to predictor values, which are recorded in the predictor matrix $\X$. The set of split rule candidates, denoted $\mathcal{C}$, is defined by $\X$ with each element indexed as $c_{jk}$ where $j = 1, \dots p$ indexes a variable (column) of $\X$ and $k$ indexes a set of candidate cutpoints (row) of $\X$. In this paper, we use the term \textit{cutpoint} to refer to $c_{jk}$, with the understanding that the value of $c_{jk}$ and the index $j$ jointly define a splitting rule.

Let $|\mathcal{C}|$ denote the total number of cutpoint candidates. Let $\Phi$ denote prior hyper-parameters and $\Psi$ denote model parameters; that is, $\Phi$ is fixed during the entire process fitting (such as prior parameters), while model parameters $\Psi$ are to be estimated (such as residual variance for the regression case)\footnote{The distinction between $\Phi$ and $\Psi$ is perhaps clearest in concrete instances of the algorithm; see section \ref{sec:gaussian} for an example of Gaussian regression.}.

Consider a likelihood $\ell(\y_b; \mu_b, \Psi_b)$ \textit{on one leaf} with vector of data observations within the leaf $\y_b$, leaf-specific parameter $\mu_b$ and additional model parameters $\Psi_b$. In the following text, we omit subscript $b$ for simplicity. The leaf parameter $\mu$ is given a prior $\pi(\mu ; \Phi)$, which induces a prior predictive distribution
\begin{equation}\label{eqn:marginallikelihood}
	m\left(\y ; \Phi, \Psi\right) := \int \ell \left(\mathrm{y}; \mu, \Psi \right) \pi\left(\mu ; \Phi\right) d\mu,
\end{equation}
This prior predictive distribution will define the XBART split criterion. Observe that this framework can accommodate many different models, defined by the choice of likelihood function $\ell(\mathrm{y}; \mu, \Psi)$: it could be Gaussian for regression (Section 3), a multinomial distribution for classification, etc.

Intuitively, the split criterion arises via Bayesian estimation of a single unknown parameter, the cutpoint defining the partition (which may be null, as in the data are not to be split at all). From this ``local'' perspective (ignoring its implications further down the tree), the posterior probability of each candidate cutpoint is determined according to the above prior predictive distribution (equivalently, the marginal likelihood). In more detail, a cutpoint $c_{jk}$ partitions the current node to left and right child nodes, with (sub)vectors $\y_{jk}^{(1)}$ and $\y_{jk}^{(2)}$, respectively. Assuming that observations in separate leaf nodes are independent, the joint prior predictive associated to this local Bayesian model is simply the product of the predictive distribution in each of the two partitions defined by $c_{jk}$:
\begin{equation}\label{eqn:general_split_criterion}
	L\left(c_{jk}\right) := m\left(\y_{jk}^{(1)} ; \Phi, \Psi \right)\times m\left(\y_{jk}^{(2)} ; \Phi, \Psi\right),
\end{equation}
which defines the split criterion for cutpoint $c_{jk}$. Similarly, the {\em null cutpoint} is defined as
\begin{equation}\label{eqn:general_split_criterion_nosplit}
	L\left(\emptyset\right) := |\mathcal{C}| \left(\frac{(1+d)^\beta}{\alpha} - 1\right) m\left(\y ; \Phi, \Psi\right)
\end{equation}
where $d$ is depth of the current node and $\alpha$, $\beta$ are hyper-parameters. Once the split criterion has been evaluated at all candidate cutpoints (including the null cutpoint), one of them is sampled with probability proportional to the split criterion value:
\begin{equation}\label{eqn:sample}
\begin{split}
	P(c_{jk}) &= \frac{L({c}_{jk})}{\sum_{c_{jk}\in \mathcal{C}} L({c}_{jk}) + L(\emptyset)}, \\ 
	P\left(\emptyset\right) &= \frac{ L(\emptyset)}{\sum_{c_{jk}\in \mathcal{C}} L({c}_{jk}) + L(\emptyset)}.
\end{split}
\end{equation}
The additional weight on the null cutpoint criterion, $ |\mathcal{C}| \left(\frac{(1+d)^\beta}{\alpha} - 1\right)$, was chosen to match the tree prior used in standard BART \citep{chipman2010bart}, in the following sense. Setting $m(\blank ; \Phi, \Psi) = 1$ for all cutpoints in $\mathcal{C}$ entails that the prior probability of splitting at one of the cutpoints is  $$1 - P(\emptyset) = \frac{|\mathcal{C}|}{|\mathcal{C}|\left(\frac{(1+d)^\beta}{\alpha} - 1\right) + |\mathcal{C}|}  = \alpha (1 + d)^{-\beta}.$$ 
Once the null cutpoint is sampled (or other stopping conditions are met), the recursion terminates, returning the leaf parameter $\mu$, sampled from the associated ``local'' Bayesian posterior. Algorithm \ref{alg:GFR} presents this algorithm, which we call {\sc GrowFromRoot}.  Note that Algorithm \ref{alg:GFR} expresses $m(\y ; \Phi, \Psi)$ in terms of a sufficient statistic $s(\y)$, which can aid computation; for a concrete example see Section \ref{sec:sample}.



\begin{algorithm*}[!htbp]
	\small
	\caption{GrowFromRoot}\label{alg:GFR}
	\begin{algorithmic}[1]
		\Procedure{GrowFromRoot}{$\y, \X,\Phi,\Psi$, $d$, $T$, {\tt node}} \\
		\textbf{outcome} Modifies $T$ by adding nodes and sampling associated leaf parameters $\boldsymbol{\mu}$.
		\If{the stopping conditions are met} \State Go to step 15, update leaf parameter $\mu_{\mbox{{\tt node}}} $.
		\EndIf
		\State $s^{\emptyset} \gets s(\y, \X, \Psi, \mathcal{C}, \mbox{all})$. \Comment{Compute sufficient statistic of stop-splitting.}
		\For{$c_{jk} \in \mathcal{C}$}   \Comment{Loop over all cutpoint candidates.}
		\State $s_{jk}^{(1)} \gets s(\y, \X, \Psi, \mathcal{C}, j,k, \mbox{left})$.\Comment{Compute sufficient statistic of left candidate node.}
		\State $s_{jk}^{(2)}\gets s(\y, \X, \Psi, \mathcal{C}, j,k, \mbox{right})$.\Comment{Compute sufficient statistic of right candidate node.}
		\State Calculate $L(c_{jk}) = m\left(s_{jk}^{(1)} ; \Phi, \Psi \right)\times m\left(s_{jk}^{(2)} ; \Phi, \Psi\right)$
		\EndFor
		\State Calculate $L(\emptyset) = |\mathcal{C}| \left(\frac{(1+d)^\beta}{\alpha} - 1\right) m\left(s^{\emptyset} ; \Phi, \Psi\right).$
		\State Sample a cutpoint $c_{jk}$ proportional to integrated likelihoods 
		$$P(c_{jk}) = \frac{L({c}_{jk})}{\sum_{c_{jk}\in \mathcal{C}} L({c}_{jk}) + L(\emptyset)},$$
		$$ \text{or } P(\emptyset) = \frac{ L(\emptyset)}{\sum_{c_{jk}\in \mathcal{C}} L({c}_{jk}) + L(\emptyset)} \text{ for the null cutpoint.} $$
		\If{the null cutpoint is selected}
		\State $\mu_{\mbox{{\tt node}}} \gets$ SampleParameters($s^{\emptyset}$)
		\State \textbf{return}.
		\Else
		\State Create two new nodes, {\tt left\_node} and {\tt right\_node}, and grow $T$ by designating them as the current node's ({\tt node}) children.
		\State Partition the data $(\y, \X)$ into left $(\y_{\mbox{left}}, \X_{\mbox{left}})$and right $(\y_{\mbox{right}}, \X_{\mbox{right}})$ parts, according to the selected cutpoint $x_{ij'} \leq  x^*_{kj}$ and $x_{ij'} > x^*_{kj}$, respectively, where $x^*_{kj}$ is the value corresponding to the sampled cutpoint $c_{jk}$.
		\State  \Call{GrowFromRoot}{$\y_{\mbox{left}}, \X_{\mbox{left}},\Phi,\Psi$, $d+1$, $T$, {\tt left\_node}}.
		\State  \Call{GrowFromRoot}{$\y_{\mbox{right}}, \X_{\mbox{right}},\Phi,\Psi$, $d+1$, $T$, {\tt right\_node}}.
		\EndIf
		\EndProcedure
	\end{algorithmic}
\end{algorithm*}



\subsection{Tree ensembles}
Tree ensemble prediction methods combine $L$ decision or regression trees, $T_l$, $l = 1, \dots, L$ to produce a prediction function mapping $\mathcal{X} \rightarrow \mathcal{Y}$. Although the simplicity of single-tree methods such as CART have their appeal, the most accurate tree-based prediction methods are ensemble approaches: Random forests \citep{breiman2001random}, gradient boosting \citep{friedman2001greedy}, and BART \citep{chipman2010bart}. To extend the grow-from-root algorithm to tree ensembles (or ``forests'') we simply define the split criterion itself as a function of the ``leave-one-out forest''. Let $\mathcal{F} = \lbrace T_1, \dots, T_L \rbrace$ denote the complete forest and let $\mathcal{F}_{\minus h} = \mathcal{F} \setminus T_h$. Similarly, let $\mathcal{M} = \lbrace \boldsymbol{\mu}_1, \dots , \boldsymbol{\mu}_L \rbrace$ denote the set of associated leaf parameters and $\mathcal{M}_{\minus h} = \mathcal{M} \setminus \boldsymbol{\mu}_h$. One may then define a marginal likelihood as
\begin{equation}
	m\left(\y ; \Phi,  \lbrace \Psi, \mathcal{F}_{\minus h} \mathcal{M}_{\minus h}\rbrace\right) := \int \ell \left(\mathrm{y}; \mu, \lbrace \Psi, \mathcal{F}_{\minus h} \mathcal{M}_{\minus h}\rbrace\right) \pi(\mu ; \Phi) d\mu,
\end{equation}
provided the integral is well-defined. In this formulation, {\sc GrowFromRoot} may be conceptualized in terms of a Bayesian model with all but the $h$-th tree  known (including the associated leaf parameters). This is much like specifying a model in terms of full conditional distributions; in general this will not yield a stationary distribution, but in Section \ref{stationarity} we show that XBART does. In the case of additive ensembles for mean regression with additive Gaussian errors, these expression are particularly convenient (see Section \ref{sec:gaussian}). 

The XBART stochastic ensemble method is detailed in Algorithm \ref{alg:XBART}. It produces $I$ samples of the forest $\mathcal{F}$. We refer to one iteration of the algorithm, sampling all $L$ trees once, as a \textit{sweep}.  Additional (model-specific) parameters, $\Psi$, may be updated in between sampling each tree (for a total of $L$ updates per sweep), or in between sampling each forests (one update per sweep).

%
%
%

\begin{algorithm*}[!htbp]
	\small
	\caption{Accelerated Bayesian Additive Regression Trees (XBART)}\label{alg:XBART}
	\begin{algorithmic}[1]
		\Procedure{XBART}{$\y, \X, \Phi, L, I$} \\
		\textbf{output} $I$ posterior draws of a forest (and associated leaf parameters) comprising $L$ trees.
		\State Initialize $\Psi$ and all $T_h$, for $h = 1, \dots, L$.
		\For{$iter$ in 1 to $I$}
		\For{$h$ in 1 to $L$}
		\State Create {\tt new\_node}.
		\State Initialize tree $T_{h}^{(iter)}$ to the root node.
		\State  {\sc GrowFromRoot}$\left ( \y, \X,\lbrace \Psi, \mathcal{F}_{\minus h}, \mathcal{M}_{\minus h} \rbrace,\Phi, T_{h}^{(iter)},\text{{\tt new\_node}}  \right )$.
		\State Sample some elements of $\Psi$, in between of each tree.
		\EndFor
		\State Sample some elements of $\Psi$, in between of each sweep.
		\EndFor
		\EndProcedure
	\end{algorithmic}
\end{algorithm*}

%
%

\section{Regression with XBART}\label{sec:gaussian}

This section derives the specific forms of the split criterion and the parameter sampling distributions corresponding to a homoskedastic Gaussian additive error model:
\begin{equation}\label{eqn:gaussian_model}
\begin{split}
	Y &= f(\x) + \epsilon,\\
	&= \sum_{l = 1}^L g\left(\x ; T_l, \boldsymbol{\mu}_l\right) + \epsilon,
\end{split}
\end{equation}
where $f$ is a unknown mean function that is represented as a sum of regression trees, $f(\x) = \mathbb{E}\left[Y \mid\x\right]$, and $\epsilon \sim \N(0, \sigma^2)$. The additional non-tree parameters in this case are the residual variance and the common prior variance over the leaf means: $\Psi := (\sigma^2, \tau)$, which are given inverse-Gamma$(a_\sigma, b_\sigma)$  and inverse-Gamma$(a_\tau, b_\tau)$ priors, respectively. Parameter $\sigma^2$ is updated between tree updates while $\tau$ is updated between sweeps.

Reviewing notation, $\x_i = (x_{i1} \cdots, x_{ip})^t$ is the $i$-th observation of a $p$ dimensional covariate (row) vector and $y\in \mathbb{R}$ is the observed response variable. Capital letter $Y$ denote the response considered as a random variable, while $\y = (y_1, \cdots, y_n)^t$ denotes a length $n$ vector of corresponding realized observations, and $\X = (\x_{1}, \cdots, \x_{n})^t$ is a $n\times p$ matrix of covariate data, where rows are observations and columns are features. Leaf parameters are given independent and identical Gaussian priors, $\mu \sim \N(0, \tau)$. In the notation of section \ref{sec:gfr}, these modeling choices correspond to hyper-parameter $\Phi = \lbrace a_\sigma, b_\sigma, a_\tau, b_\tau \rbrace$ and model parameter $\Psi = \lbrace \sigma, \tau \rbrace $, respectively. 

\subsection{Sampling cutpoints}\label{sec:sample}
This section derives the explicit form of the split criterion for the above Gaussian regression model, corresponding to lines 6 through 13 in Algorithm \ref{alg:GFR}.

\begin{proposition}
For the Gaussian regression case, the split criterion for a cutpoint candidate $c_{jk}$ and the null cutpoint in equation (\ref{eqn:general_split_criterion}) and (\ref{eqn:general_split_criterion_nosplit}) has the form 
\begin{equation}\label{eqn:split_gaussian}
	\begin{aligned}
		& L(c_{jk})  \propto \exp\left[\frac{1}{2}\sum_{b=1}^2 \left(\log{ \left( \frac{\sigma^2}{\sigma^2 + \tau n_{jk}^{(b)}} \right)}  + \frac{\tau}{\sigma^2(\sigma^2 + \tau n_{jk}^{(b)})} \left(s_{jk}^{(b)}\right)^2 \right)\right],\\
		& 	L(\emptyset) \propto |\mathcal{C}| \left(\frac{(1+d)^\beta}{\alpha} - 1\right) \times \exp\left[\frac{1}{2}\left(\log{ \left( \frac{\sigma^2}{\sigma^2 + \tau n} \right)}  + \frac{\tau}{\sigma^2(\sigma^2 + \tau n)} s(\y)^2\right)\right].
	\end{aligned}
\end{equation}
Here, suppose there are $n$ data observations in the current node, and the cutpoint candidate $c_{jk}$ partitions data to left and right child nodes, with $n_{jk}^{(1)}$ and $n_{jk}^{(2)}$ observations in each, and $n_{jk}^{(1)} + n_{jk}^{(2)} = n$. Let $s_{jk}^{(1)} = \sum_{i \in \text{left child}} y_i, s_{jk}^{(2)} = \sum_{i \in \text{right child}} y_i$ denote the sufficient statistics of each child node; let $s(\y) = s_{jk}^{(1)} + s_{jk}^{(2)} = \sum_{i=1}^n y_i$ denote the sufficient statistic of the unpartitioned data. The notation $|\mathcal{C}|$ indicates the number of cutpoint candidates under consideration.
\end{proposition}

\begin{remark}
	Expressions (\ref{eqn:split_gaussian}) are denoted up to proportionality because multiplicative constants cancel when normalizing the probabilities as in equation (\ref{eqn:sample}). Also, note that because the XBART regression split criterion depends only on a univariate statistic, it can be computed rapidly if the data are pre-sorted; see Appendix \ref{sec:computation} for details regarding computational considerations.
\end{remark}

\begin{proof}
We first describe the marginal likelihood criteria for a single tree. Assuming that observations within the same leaf node are independent and identically distributed, the prior predictive distribution of equation (\ref{eqn:general_split_criterion}) is simply a mean-zero multivariate Gaussian distribution,
\begin{equation}\label{eqn:integrate_leaf_gaussian}
\begin{aligned}
	m(\y ; \tau, \sigma) &= \int \phi \left(\y ; \mu \mathrm{J}_n, \sigma^2\mathbf{I}_n\right) \phi \left(\mu; 0, \tau\right) d\mu = \phi\left(0, \Omega\right)\\
	&= (2\pi)^{-n/2}\det(\Omega)^{\minus 1/2} \exp{\left( -\frac{1}{2} \y^t \Omega\inv \y \right)},
\end{aligned}
\end{equation}
with covariance matrix and corresponding inverse
\begin{equation*}
	\Omega = \tau \mathrm{J}_n\mathrm{J}_n^t + \sigma^2 \mathbf{I}_n, \quad \Omega^{-1}  =  \sigma^{-2}\mathbf{I}_n - \frac{\tau}{\sigma^2(\sigma^2 + \tau n)} \mathrm{J}_n\mathrm{J}_n^t.
\end{equation*}
Here $\mathrm{J}_n$ is a length $n$ column vector of all ones; $\mathbf{I}_n$ is a $n$ dimensional identity matrix; $\phi\left(\y ; \mu\mathrm{J}_n, \sigma^2\mathbf{I}_n\right)$ denotes the likelihood, the density of multivariate Gaussian distribution with mean $\mu \mathrm{J}_n$ and covariance matrix $\sigma^2\mathbf{I}_n$; $\phi(\mu; 0, \tau)$ is the density of the univariate Gaussian prior over $\mu$; and $n$ is number of data observations comprising $\y$.

The prior predictive density of equation (\ref{eqn:integrate_leaf_gaussian}) may be simplified as follows. First, applying Sylvester's determinant theorem  yields
\begin{equation*}
	\det{\Omega^{-1}}= \sigma^{-2n} \left (1 - \frac{\tau n}{\sigma^2 + \tau n} \right) = \sigma^{-2n} \left (\frac{\sigma^2}{\sigma^2 + \tau n} \right).
\end{equation*}
Taking logarithms of the density in equation (\ref{eqn:integrate_leaf_gaussian}) yields
\begin{equation*}
	\begin{aligned}
		-\frac{n}{2} & \log{(2\pi)}-n\log\sigma  -\frac{1}{2} \frac{\y^t\y} {\sigma^{2}} + \frac{1}{2} \log{ \left( \frac{\sigma^2}{\sigma^2 + \tau n} \right)} + \frac{1}{2}\frac{\tau}{\sigma^2(\sigma^2 + \tau n)}s(\y)^2 ,
	\end{aligned}
\end{equation*}
with
$$s(\y) = \y^t\mathrm{J} = \sum_{i=1}^n y_i,\quad s(\y)^2 = \y'\mathrm{J}\mathrm{J}^t \y = \left(\sum_{i=1}^n y_i\right)^2.$$ Cutpoint $c_{jk}$ partitions the data into left and right child nodes according to the split rule $\{x_j \leq c_{jk}\}$ and $\{x_j > c_{jk}\}$. Denote the response (sub)vectors in left and right child nodes as $\y_{jk}^{(1)}$, $\y_{jk}^{(2)}$ of length $n_{jk}^{(1)}$ and $ n_{jk}^{(2)}$, respectively. The (log) marginal likelihood for cutpoint $c_{jk}$ is then 
\begin{equation}\label{loglikelihood}
\footnotesize
	\begin{aligned}
		&\log\left[m\left(s_{jk}^{(1)}; \tau, \sigma\right)\times m\left(s_{jk}^{(2)}; \tau, \sigma\right)\right]\\                                                                                                         
		=& \sum_{b=1}^2 -\frac{n_{jk}^{(b)}}{2}\log(2\pi) - n_{jk}^{(b)}\log\sigma- \frac{1}{2} \frac{\left(\y_{jk}^{(b)}\right)^t \y_{jk}^{(b)}}{\sigma^2}  + \frac{1}{2} \log\left(\frac{\sigma^2}{\sigma^2 + \tau n_{jk}^{(b)}}\right)  + \frac{1}{2}\frac{\tau}{\sigma^2\left(\sigma^2 + \tau n_{jk}^{(b)}\right)}\left(s_{jk}^{(b)}\right)^2\\
		= & -\frac{n}{2} \log{(2\pi)} - n \log{(\sigma)} - \frac{1}{2} \frac{\y^t\y} {\sigma^{2}} + \frac{1}{2} \sum_{b=1}^2 \left[\log{ \left( \frac{\sigma^2}{\sigma^2 + \tau n_{jk}^{(b)}} \right)}  + \frac{\tau}{\sigma^2\left(\sigma^2 + \tau n_{jk}^{(b)}\right)} \left(s_{jk}^{(b)}\right)^2 \right],
	\end{aligned}
\end{equation}
where the summation runs over the two child nodes partitioned by $c_{jk}$, and $s_{jk}^{(1)}$ and $s_{jk}^{(2)}$ denote the sufficient statistics of the respective partitions, specifically
\begin{equation*}
	s_{jk}^{(1)} = s\left(\y_{jk}^{(1)}\right), \quad s_{jk}^{(2)} = s\left(\y_{jk}^{(2)}\right), \quad s_{jk}^{(1)} + s_{jk}^{(2)} = \sum_{i=1}^n y_i = s\left(\y\right).
\end{equation*}
Disregarding the first three terms of equation (\ref{loglikelihood}), which do not vary for different cutpoints, the split criterion for cutpoint $c_{jk}$ can be expressed compactly as:
\begin{equation*}
	\begin{aligned}
		& L(c_{jk})  = m\left(s_{jk}^{(1)}; \tau, \sigma\right)\times m\left(s_{jk}^{(2)}; \tau, \sigma\right)                                                                                                               \\
		          & \propto  \exp\left[\frac{1}{2}\sum_{b=1}^2 \left(\log{ \left( \frac{\sigma^2}{\sigma^2 + \tau n_{jk}^{(b)}} \right)}  + \frac{\tau}{\sigma^2\left(\sigma^2 + \tau n_{jk}^{(b)}\right)} \left(s_{jk}^{(b)}\right)^2 \right)\right].
	\end{aligned}
\end{equation*}
For future reference, we denote the logarithm of the split criterion by $l(c_{jk}) = \log\left(L(c_{jk})\right)$. Similarly, following equation (\ref{eqn:general_split_criterion_nosplit}), the null cutpoint criterion is calculated according to
\begin{equation*}
	L(\emptyset) \propto |\mathcal{C}| \left(\frac{(1+d)^\beta}{\alpha} - 1\right) \times \exp\left[\frac{1}{2}\left(\log{ \left( \frac{\sigma^2}{\sigma^2 + \tau n} \right)}  + \frac{\tau}{\sigma^2(\sigma^2 + \tau n)} s(\y)^2\right)\right].
\end{equation*}
\end{proof}

Finally, observe that the XBART split criterion involves the (current estimate of the) residual standard error $\sigma$, thereby providing adaptive regularization. The role of $\sigma$ in the split criterion vanishes asymptotically (see Section \ref{equivalence} expression \ref{eqn:theoretical_split}), but the algorithm's finite sample performance depends on the specific value.

\subsubsection{The ensemble case}
In the case of multiple trees, XBART  ``residualizes'' the data with respect to the partial fit corresponding to the partial forest $\mathcal{F}_{\minus h}$. Formally,

\begin{equation}
\begin{split}
	m(\y ; \tau, \lbrace \sigma, \mathcal{F}_{\minus h}, \mathcal{M}_{\minus h} \rbrace) &= \int \phi \left(\y ;  f_{\minus h}(x) + \mu \mathrm{J}_n, \sigma^2\mathbf{I}_n\right) \phi \left(\mu; 0, \tau\right) d\mu \\
	& = m\left(\y - f_{\minus h}(\X) ; \tau, \sigma\right)
\end{split}
\end{equation}
where 
\begin{equation*}
f_{\minus h}(\X) = \sum_{l \neq h} g\left(\X ; T_l, \boldsymbol{\mu}_l\right)
\end{equation*}
and $g(\cdot)$ is applied row-wise to $\X$. According to this specification, the split criterion may be computed as described above, simply replacing response $\y$ by the residual $\y - f_{\minus h}(\X)$ in expressions (\ref{eqn:split_gaussian}). Such ``residualization'' is analogous to the ``Bayesian back-fitting'' procedure described in \cite{chipman2010bart}. Appendix \ref{sec:demo:XBART} visualizes a simple example of fitting and residualization of a three-tree forest.

\subsection{Parameter sampling}\label{sec:sampling}
This section introduces sampling steps for leaf-specific parameters as well as global model parameters for the Gaussian regression model, corresponding to step 15 in Algorithm \ref{alg:GFR} and lines 9 and 11 in Algorithm \ref{alg:XBART}. 
\subsubsection{Leaf parameter sampling}
In GrowFromRoot (Algorithm \ref{alg:GFR}) , after the null cutpoint is sampled, or other stopping conditions are satisfied, the current partition $\mathcal{A}_{lb}$ is designated as the $b$-th leaf of $l$-th tree and its associated leaf parameter $\mu_{lb}$ is sampled in step 15. Assuming a conjugate Gaussian prior, $\mu_{lb} \sim \N(0, \tau)$, yields the following conjugate ``full conditional'' (given the current partial fit $f_{\minus h}$):
\begin{equation*}
	\mu_{lb} \sim \N\left(\frac{ s_{lb}}{\sigma^2 \left(\frac{1}{\tau} + \frac{n_{lb}}{\sigma^2}\right) }, \frac{1}{\frac{1}{\tau} + \frac{n_{lb}}{\sigma^2}}\right),
\end{equation*}
where $n_{lb}$ is number of data observations in the node and $s_{lb} = \sum_{\x_i \in \mathcal{A}_{lb}} \left ( y_i - f_{-h}(\x_i) \right )$.

\subsubsection{Global parameter sampling}
Next, consider the global (non-tree-specific) model parameter sampling step (lines 9 and 11) of Algorithm \ref{alg:XBART}. The Gaussian regression model (\ref{eqn:gaussian_model}) incorporates two global parameters, the residual variance $\sigma^2$ and the prior variance $\tau$ over the leaf mean parameter.  Tthe residual variance $\sigma^2$ is sampled in between of each tree (line 9 of Algorithm \ref{alg:XBART}). Assuming a conjugate inverse-Gamma prior, $\sigma^2 \sim \text{inverse-Gamma}(a_{\sigma}, b_{\sigma})$,  the corresponding ``full conditional'' is
\begin{equation*}
	\sigma^2 \sim \text{inverse-Gamma}\left(n + a_{\sigma}, \mathrm{r}^t \mathrm{r} + b_{\sigma}  \right),
\end{equation*}
where $\mathrm{r} = \y - f(\X)$ is the vector of residuals based on the current fit
\begin{equation*}
	f(\X) = \sum_{l = 1}^L g\left(\X ; T_l, \boldsymbol{\mu}_l\right),
\end{equation*}
and $g(\cdot)$ is applied row-wise to $\X$.

The prior variance of the leaf means, $\tau$, is likewise given a conjugate inverse-Gamma prior, $\tau \sim \text{inverse-Gamma}(a_\tau, b_\tau)$. Based on extensive experiments, we advise sampling $\tau$ in between sweeps (line 11 of Algorithm \ref{alg:XBART}) rather between each tree sampling step. Let $\mu_{lb}$ denote the $b$-th leaf of the $l$-th tree, $\tilde{\boldsymbol{\mu}} = \{\mu_{lb}\}_{1\leq l \leq L, 1\leq b \leq B_l}$ denote the collection of \textit{all} leaf parameters of the $L$ trees in the forest, and $|\tilde{\boldsymbol{\mu}}|$ denote the total number leaf nodes in the ensemble. Then $\tau$ is updated according to 
\begin{equation*}
\tau \sim \text{inverse-Gamma}\left( |\tilde{\boldsymbol{\mu}}| + a_\tau, \sum_{\mu_{lb} \in \tilde{\boldsymbol{\mu}}} \mu_{lb}^2+ b_\tau \right).
\end{equation*}
Based on extensive simulation studies, we recommend setting $a_\tau = 3$ and $b_\tau = 0.5 \times \text{Var}(\y) / L$ as default choices.




\subsection{Prediction}
Predictions are obtained from XBART by taking posterior point-wise averages as if the sampled trees were draws using a standard Bayesian Monte Carlo algorithm. That is, given $I$ iterations of the algorithm, the final $I - I_0$ samples are used to compute a point-wise average function evaluation, where $I_0 < I$ denotes the length of the burn-in period. We recommend $I = 40$ and $I_0 = 15$ for routine use. The final estimator is therefore expressed as
\begin{equation*}
	\hat{y}_i = \bar{f}(\x_i) = \frac{1}{I-I_0}\sum_{k > I_0}^I f^{(k)}(\x_i),
\end{equation*}
where $f^{(k)}$ denotes the $k$-th iteration of Algorithm \ref{alg:XBART}. This point-wise mean estimator corresponds to the Bayes optimal estimator under mean squared estimation loss if we were to regard our samples as coming from a traditional posterior distribution. As the {\sc GrowFromRoot} sampler is not a proper full conditional, this estimator must be considered an approximation of some sort. Nonetheless, simulation results strongly suggest that the approximation is adequate, as described in the following section. Importantly, the recursive nature of XBART enables us to employ many computational strategies that cannot be applied to BART MCMC; see the Appendix \ref{sec:computation} for a description of these strategies.  Section \ref{sec:warm_start} describes how to use XBART to improve the posterior exploration of BART MCMC for fully Bayesian inference and Section \ref{sec:theory} provides some preliminary theory pertaining to the above XBART point estimator.

\section{Demonstrations}\label{sec:example}
This section documents the favorable empirical performance of XBART relative to other popular nonlinear regression methods. Results on synthetic data are presented first, followed by results on a number of publicly available real-world data sets.

Here, our comparisons focus on XGBoost  \citep{chen2016xgboost} and random forests \citep{breiman2001random}. These two methods were chosen because of their immense popularity and because in our experience they seem to work well in different regimes: XGBoost performs best in situations with low noise but complex mean functions, while random forests performs best in situations with substantial noise (sources of unmeasured variation). A key finding of our simulations is that XBART performs well across this spectrum, making it a strong default choice when the quality of one's data is unknown in advance. Additional comparisons with neural networks, using Keras \citep{chollet2015keras}, are included in Appendix \ref{sec:appendix:simulation} along with implementation details. A good faith effort was made to cross-validate XGBoost and random forests thoroughly and efficiently, but additional improvements may be possible in the hands of a skilled user. All-the-same, a key practical advantage of XBART is a relatively automated model-fitting process. Although XBART itself has numerous tuning parameters --- the number of trees, the number of sweeps, prior parameters, and the number of cutpoint candidates --- all of the simulation results reported below are based on a {\em single} set of default parameters. 

\subsection{Simulation studies}\label{regression_simulation}
The goal of these simulation studies is to examine the behavior of XBART regression across a variety of data generating processes (DGPs). Although no simulation study can be truly exhaustive, by varying several individual aspects of the data generating process, a performance profile emerges that suggests that XBART is a supervised learning algorithm of wide applicability.  Specifically, function estimation at a set of hold-out locations was judged according to root mean squared error (RMSE). 

All data generating processes are homoskedastic additive error models:
\begin{equation*}
	Y = f(\x) + \epsilon
\end{equation*}
with $\E(\epsilon) = 0$. Within this framework many individual elements are varied: the mean function $f$, the error distribution, including its shape and its variance, the predictor variable matrix $\X$, specifically the number of individual features and the dependence between them, as well as the size of the training sample.\\

\noindent{{\em Mean functions}}\\
Four mean functions are considered, as defined in Table \ref{tab:truef}. The selection of these functions was intended to cover a range of important special cases: linearity, additive models, models with interactions, nonlinear smooth functions, and functions with discontinuities. The value of using fixed functions as opposed to randomly generated polynomials, say, is that we can study coverage under repeated sampling. These functions were designed to be easily understandable to a human, but challenging to learn.\\

\begin{table}[!htbp]
\small
	\def\arraystretch{1}
	\centering 
	\begin{tabular}{ll} 
		\toprule
		Name         & Function                                                                                                \\
		\hline
		Linear       & $ \x^t \mathrm{\gamma} $;\; $ \gamma_j = -2+ \frac{4(j-1)}{d-1} $                                       \\ 

		Single index & $10\sqrt{a} + \sin{(5a)}$;\;$a=\sum_{j=1}^{10} (x_j - \gamma_j)^2$;\; $\gamma_j = -1.5+ \frac{j-1}{3}$. \\


		Trig+poly  & $5\sin(3x_1)+2x_2^2 + 3x_3x_4 $                                                                         \\

		Max          & $\max(x_1,x_2,x_3) $                                                                                    \\ 
		\bottomrule
	\end{tabular}
	\caption{Four true $f$ functions}
	\label{tab:truef}
\end{table}

\noindent{{\em Predictor variables}}\\
The $n\times p $ predictor matrix $\mathbf{X}$ is generated in one of two ways:
\begin{enumerate}
	\item Independent variables: each element of $\mathbf{X}$ is drawn independently from a standard Gaussian distribution.
	\item Correlated predictors with factor structure: each row of $\X$ is drawn from a Gaussian factor model with $k = p / 5$ factors. Latent factor scores are drawn according to $\mathbf{F}_{k\times n} \sim \N(0,1)$. The factor loading matrix, $\mathbf{B}_{p\times k}$, has entries that are either zero or one, with exactly five ones in each column and a single 1 in each row, so that $\mathbf{B}\mathbf{B}^t$ is block diagonal, with blocks of all ones and all other elements being zero. The regressors are then set as $\X = (\mathbf{B}\mathbf{F})^T + \boldsymbol{\varepsilon}$ where $\boldsymbol{\varepsilon}$ is a $n \times p$ matrix of errors with independent $\N(0, 0.01k)$ entries. Finally, each column of $\mathbf{X}$ is scaled to standard deviation 1.
\end{enumerate}

\noindent{{\em Error distribution}}\\
The error term $\epsilon$ is drawn in one of two ways:
\begin{enumerate}
	\item Gaussian. Draw $\epsilon_i \iid \N(0, \sigma^2)$ and  $\sigma^2 = \kappa^2 \mbox{Var}(f)$ where $\kappa$ controls the signal-to-noise ratio. 
	\item Student-$t$. Draw $\epsilon_i \iid \kappa\sqrt{\mbox{Var}(f)} \times \left(t_3   / \sqrt{3}\right)$, a student-$t$ distribution with degree of freedom 3. Note that the additional scaling factor $\sqrt{3}$ ensues variance of $\epsilon_i$ to be $\kappa^2\mbox{Var}(f)$.
\end{enumerate}
In each of these two cases, we consider $\kappa \in \lbrace 1, 10 \rbrace$.\\

\noindent{{\em Sample sizes and number of features ($n$ and $p$)}}\\
Sample size and the number of features were considered in a variety of combinations and are reported in Table \ref{tab:pn}. Large and small sample sizes, including the case of more predictors than observations, are considered. \\

\noindent{{\em Hyper-parameter setting}}\\
The simulation studies use the default hyper-parameter settings suggested in section \ref{sec:sampling}. 

\begin{table}[]
\begin{center}
\small
\begin{tabular}{p{26mm} | p{15mm} | p{15mm} | p{22mm}}
\centering $p = 30$   & \centering $p = 100$ & \centering $p = 500$ & $p = 1{,}000$ \\ \hline   
 &&&\\
 $n = \setstretch{1.1} \begin{cases} 10{,}000,\\  50{,}000,\\ 250{,}000 \end{cases}$  & $n = 1{,}000$ & $n = 300$ & $n = \begin{cases} 500,\\  1{,}000\end{cases}$ \\
\end{tabular}\caption{Sample sizes and number of predictors considered in the simulation study.}\label{tab:pn}
\end{center}
\end{table}


\subsubsection{Results}

Complete simulation results are reported in tabular form in Appendix \ref{sec:appendix:simulation}. Here we summarize our findings, which are visualized in Figure \ref{fig:simu2} consisting of six panels. The first column is the low-noise setting, with $\kappa = 1$; the second column is the high-noise setting, with $\kappa = 10$. The first two rows are for Gaussian errors with independent versus dependent predictors; the third row considers independent predictors with Student-t errors. Each point in the figure corresponds to a single realization (training and test set) from a particular data generating process, with sample sizes distinguished by character and true mean functions distinguished by color (see legend). In order to simultaneously compare three methods, the performance of random forests is used as a baseline, the horizontal and vertical axes represent the root mean square error of XBART and XGBoost, respectively, as a proportion of the root mean square error of random forests. Thus, the regions of the plot have the following interpretations. Points above the dashed line at 1 indicate that XGBoost performed worse than random forests; points to the right of the vertical dashed line at 1 indicate that XBART performed worse than random forests; points above the diagonal indicate that XGBoost performed worse than XBART (irrespective of the performance of random forests). Accordingly: points in the upper right quadrant (defined by the dashed lines) are cases where random forests performed best; points in the lower triangle of the unit square (with lower left corner at the origin) are cases where XGBoost performed best; points in the upper triangle of the unit square are cases where XBART performed best. For each unique color and character combination there are five independent replications.

\begin{figure}[!htbp]
	\centering
	\includegraphics[width = 0.69\textwidth]{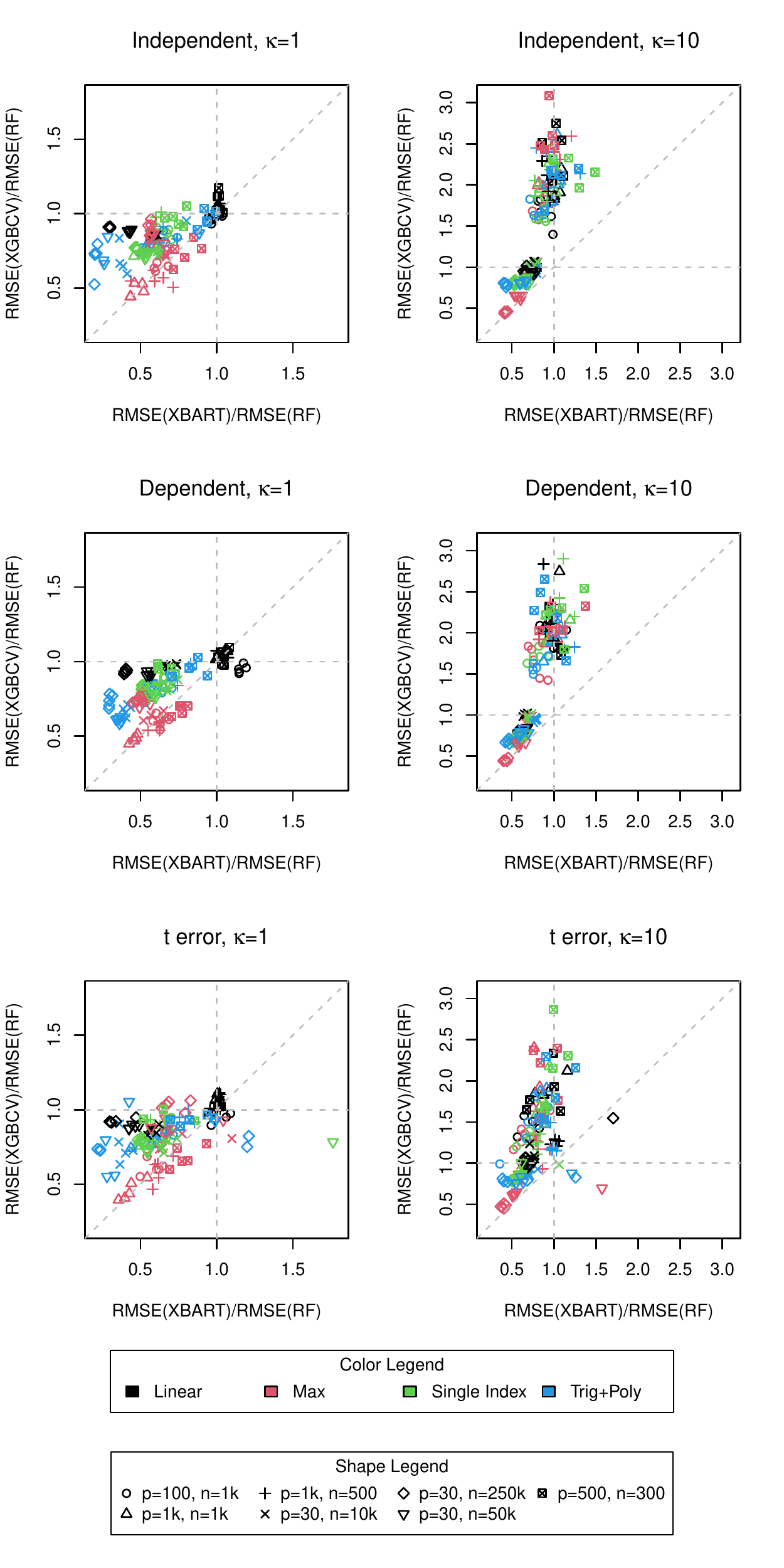}
	\caption{RMSE ratios of XBART to RF against XGBCV to RF across a variety of data generating processes. }\label{fig:simu2}
\end{figure}

Several notable patterns emerge from these plots:

\begin{enumerate}[itemsep = 0mm]
\item More points lie above the dashed horizontal line in the right column compared to the left. This means: XGBoost performs better (relatively) in the low-noise setting ($\kappa = 1$); random forests performs better (relatively) in the high-noise setting ($\kappa = 10$). 
\item More points lie above the diagonal than below it, indicating that, overall, XBART performs better than XGBoost in both the low and high noise regimes.
\item In the right column, points cluster along the vertical dashed line and the diagonal within the unit square. More specifically, in the large sample size cases (with symbols $\times$, $\triangledown$, and $\diamond$), XBART and XGBoost have similar performance, while in the more challenging low sample cases XBART and random forests have similar performance, whereas XGBoost struggles (presumably due to inadequate regularization, despite cross-validation).
\item Color-specific clusters emerge, especially in the left column. The linear function (black) clusters around the $(1,1)$ point, indicating that all three methods perform similarly on this mean function. The max function (red) clusters just below the diagonal on the unit square, meaning that XGBoost performs slightly better than XBART, while random forest struggles. 
\item In a few high sample size cases, XBART struggles with the Student-t errors as indicated by the points in the lower right.
\end{enumerate}

Not depicted in the plot are running times, which can be found in the tables in Appendix \ref{sec:appendix:simulation}. The broad trend is simply this: XBART is as fast or faster, generally speaking, than cross-validated XGBoost and typically not more than twice as slow as random forests. With broadly comparable running speed and the favorable performance profile described above, XBART appears to be a strong default choice for nonlinear regression in a wide range of settings.

\subsection{Out-of-sample prediction on empirical data}

This section compares the predictive performance of XBART and alternative methods on seven different real data sets. These data sets are accessible at the UCI machine learning data repository \citep{Dua:2019}. For each data set, 20  train-test splits are created by randomly drawing $5/6$ of the data as training sets and the remaining $1/6$ as testing sets. With $7$ data sets a total of $7\times 20 = 140$ random train-test splits were obtained.

XBART is run with default hyper-parameters as discussed in Section \ref{sec:gaussian}. The competitors are random forest and XGBoost both with and without cross-validation. The software packages and hyper-parameter settings are the same as were used in Section \ref{sec:example} and detailed in Appendix \ref{sec:appendix:simulation}.

To compare performance across different data sets, relative RMSE (RRMSE) is considered, which is the RMSE divided by the minimal RMSE for each data split:  an RRMSE of 1.0 indicates that the method achieved the minimal RMSE on a given split. Figure \ref{fig:boxplot} shows boxplots of RRMSE for each method across all data splits. Table \ref{tab:RMSE} shows the average RMSE and running time (in seconds) of each method across the 20 random splits of each data set.

\begin{table}[!htbp]
	\centering
	\resizebox{\columnwidth}{!}{%
		\begin{tabular}{lllllll}
			\toprule
			Data name (source)                               & $n$   & $p$ & XBART         & RF           & XGBCV        & XGB          \\
			\midrule
			CASP \citep{rana2015quality}                     & 45730 & 9   & 3.89 (15.5)   & 3.49 (4.0)   & 3.74 (20.5)   & 4.15 (0.4)   \\
			Energy \citep{candanedo2017data}                 & 19735 & 29  & 77.63 (14.6)  & 68.27 (1.5)  & 74.61 (18.6)  & 80.01 (0.3)  \\
			AirQuality \citep{de2008field}                   & 9357  & 13  & 45.24 (2.2)   & 45.21 (0.5)  & 45.33 (6.1)   & 45.30 (0.1)  \\
			BiasCorrection \citep{cho2020comparative}        & 7590  & 21  & 0.91 (2.3)    & 0.93 (0.5)   & 0.92 (5.1)    & 0.98 (0.1)   \\
			ElectricalStability \citep{arzamasov2018towards} & 10000 & 14  & 0.0091 (10.0) & 0.0130 (0.8) & 0.0173 (13.8) & 0.0105 (0.1) \\
			GasTurbine \citep{kaya2019predicting}            & 36733 & 9   & 0.0566 (11.3) & 0.0529 (2.8) & 0.0611 (13.6) & 0.0617 (0.2) \\
			ResidentialBuilding \citep{rafiei2016novel}      & 372   & 107 & 32.68 (0.9)   & 53.27 (0.1)  & 32.08 (6.8)   & 27.16 (0.1)  \\
			\bottomrule
		\end{tabular}
	}
	\caption{Raw RMSE and running time (in parenthesis) of all methods on different real datasets. All measurements are average of 20 independent random splits of training / testing sets.}\label{tab:RMSE}
\end{table}

\begin{figure}[h!]
	\centering
	\includegraphics[width=0.5\textwidth]{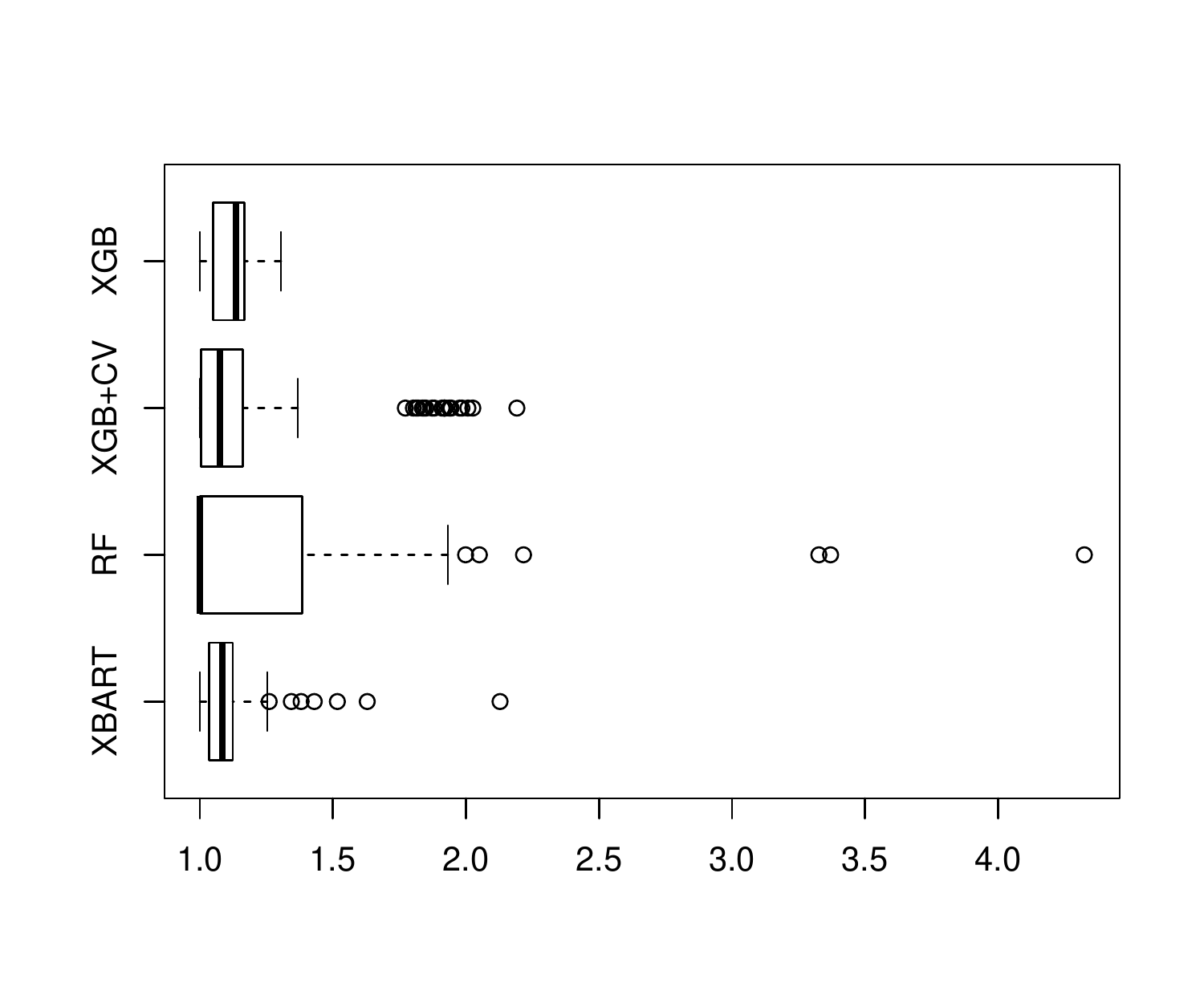}
	\caption{Boxplot of the RRMSE for each method across the total 120 training / testing splits.}\label{fig:boxplot}
\end{figure}

The upshot of Figure \ref{fig:boxplot} is that XBART tends to outperform XGBoost, with or without cross-validation. Although random forest achieves smaller average RMSE, its performance varies significantly across different data sets or data splits. Notably, cross-validated XGBoost has bigger RRMSE than XGBoost with its default settings, which suggests overfitting the training data (and, moreover, suggests that the $\kappa = 10$ simulations are more analogous to these empirical data sets). Roughly speaking, it appears that XBART performs on par with random forests on weak-signal data sets where random forests excels and performs on par with XGBoost on strong-signal data sets where XGBoost excels. This parallels the pattern observed in the simulation studies. Because in practice we often do not know which regime we are in, XBART appears to be a compelling default choice for nonlinear regression.

\section{Warm-start BART MCMC}\label{sec:warm_start}

Standard BART MCMC \citep{chipman2010bart} initializes each tree at the root (i.e., a tree only one node) and explores the posterior over trees via a random-walk Metropolis-Hastings algorithm. This approach works surprisingly well in practice, but it is natural to wonder if it takes unnecessarily long to find favorable regions in tree space. Because XBART provides a fast approximation to the BART posterior, initializing BART MCMC at XBART trees rather than null trees (with no splits) is a promising strategy to help speed convergence and accelerate posterior exploration by running multiple chains. We find that this approach yields improved point estimation and posterior credible intervals with substantially higher pointwise frequentist coverage of the mean function and a fraction of the total run time.

Consider the data generating process described in section \ref{regression_simulation}, with sample size $n =10,000$ with varying noise levels $\kappa$. We fit $40$ XBART sweeps, the first $15$ are thrown out as burn-in draws, and 25 forest draws are retained. BART was fit with a burn-in of $1,000$ samples, and $2,500$ retrained posterior samples. For the warm-start BART, 25 \textit{independent} BART MCMC chains were initialized at the $25$ forest draws obtained from XBART, and each was run for $100$ iterations without burn-in. Note that the total number of posterior draws is $2,500$, the same as the number of posterior draws by BART. We repeat drawing synthetic data, and computing intervals $100$ times. All measurements below were taken average with respect to those $100$ replications.

\begin{table}[!h]
	\centering
	\resizebox{0.8\columnwidth}{!}{%
		\begin{tabular}{l| l| lll| lllll}
			\toprule
			\multicolumn{2}{c}{} & \multicolumn{3}{c}{$\kappa = 1$} & \multicolumn{3}{c}{$\kappa = 2$}                                                  \\
			\midrule
			DGP                &     & XBART                            & BART                             & WS-BART        & XBART & BART   & WS-BART      \\
			\midrule
			\multirow{4}{*}{Linear} & coverage             & 0.78                             & 0.77                             & 0.99           & 0.50  & 0.83   & 0.98         \\
			& interval length      & 7.82                             & 6.14                             & 9.92           & 6.53  & 8.82   & 11.84        \\
			& running time         & 5.61                             & 131.17                           & 3.85  (101.86) & 3.61  & 109.43 & 3.33 (86.86) \\
			& RMSE                 & 3.11                             & 2.51                             & 1.81           & 4.74  & 4.13   & 2.53         \\
			\midrule
			\multirow{4}{*}{Max} & coverage             & 0.86                             & 0.78                             & 0.95           & 0.88  & 0.84   & 0.97         \\
			& interval length      & 0.36                             & 0.35                             & 0.46           & 0.58  & 0.64   & 0.76         \\
			& running time         & 1.57                             & 44.41                            & 1.21 (31.82)   & 1.35  & 40.23  & 1.22 (31.85) \\
			& RMSE                 & 0.11                             & 0.14                             & 0.11           & 0.17  & 0.22   & 0.17         \\
			\midrule
			\multirow{4}{*}{Single Index} & coverage             & 0.77                             & 0.73                             & 0.87           & 0.81  & 0.83   & 0.91         \\
			& interval length      & 4.84                             & 4.62                             & 5.88           & 6.81  & 7.67   & 8.49         \\
			& running time         & 5.10                             & 102.87                           & 2.97 (79.35)   & 3.92  & 90.70  & 2.81 (74.05) \\
			& RMSE                 & 1.94                             & 2.08                             & 1.92           & 2.51  & 2.73   & 2.47         \\
			\midrule
			\multirow{4}{*}{Trig+Poly}  &coverage             & 0.90                             & 0.74                             & 0.96           & 0.90  & 0.82   & 0.96         \\
			& interval length      & 3.61                             & 2.89                             & 4.23           & 5.62  & 5.06   & 6.86         \\
			& running time         & 4.68                             & 92.75                            & 3.02 (80.18)   & 3.68  & 86.81  & 2.90 (74.17) \\
			& RMSE                 & 1.03                             & 1.27                             & 1.01           & 1.65  & 1.87   & 1.60         \\
			\bottomrule
		\end{tabular}
	}
	\caption{Coverage and length of credible interval of $f$ at $95\%$ level for warm-start BART (WS-BART) MCMC. The table also shows running time (in seconds) and root mean squared error (RMSE) of all approaches. The left panel is for noise level $\kappa = 1$ and the right panel is for higher noise level $\kappa = 2$.}
	\label{warm_start_table}%
\end{table}%

Table \ref{warm_start_table} shows the credible interval coverage, length, RMSE of the point estimate, and running time of the three approaches. The running time for warm-start BART is reported as time in seconds for a single \textit{indepenent} BART MCMC, while the number in parenthesis is the running time of the entire warm-start BART fitting process, including the XBART fit and assuming all $25$ independent warm-start BART MCMC were fitted \textit{sequentially} rather than in parallel. In other words, the number in parentheses is the most conservative estimation of the total running time, because the 25 independent BART chains can be trivially parallelized to achieve much lower total running time. Indeed, with 25 processors, the run time would be the XBART run time plus the warm-start run time (not in parentheses).

Warm-start BART boasts a substantial advantage in terms of credible interval coverage and root mean squared error. In all cases, warm-start BART has the best coverage and RMSE among all three approaches and is still faster than BART under the most conservative running time estimation. When the true mean function is linear in $\x$, warm-start initialization yields considerable improvement in the estimation, which may indicate inadequate chain length of BART (that is, poor mixing).

\section{Preliminary theory}\label{sec:theory}

The XBART algorithm differs from BART or XGBoost. It is not a bagged estimator like random forests. It is not an optimization or stochastic optimization procedure, like CART or XGBoost. Neither is it a typical Markov chain Monte Carlo algorithm, like BART. Therefore, the usual frameworks for understanding supervised learning algorithms do not apply directly to XBART. On the one hand, the warm-start results in the previous section reassure us that XBART is a highly useful heuristic, even without a theoretical framework. On the other hand, an improved theoretical understanding would justify the use of XBART as a stand-alone algorithm and suggest avenues for future improvements and extensions. In this section, we prove two preliminary results about the XBART algorithm.

\begin{enumerate}
\item We establish the theory that the single-tree version of the algorithm is asymptotically consistent. Here we rely critically on recent results for CART (and random forests) and adapt the proof strategy for XBART.


\item We prove that the forest (a sum of multiple trees) version of the algorithm defines a Markov chain that has a stationary distribution.
\end{enumerate}

Taken together, these results suggest that the algorithm finds high-likelihood regions in parameter space and is not at risk of drifting aimlessly, which is consistent with the simulation evidence presented above.

\subsection{Consistency}\label{theory}
In this section, we establish the result that single-tree XBART regression is $\mathcal{L}^2$ consistent. We do not consider feature subsampling and allow each tree to grow until reaching a pre-specified maximum depth. As such, the results described in this section do not apply to the complete XBART algorithm, rather they help us to understand two key components of XBART that differentiate it from CART: stochastic sampling of the cutpoints and the use of BART's marginal likelihood split criterion.

This section adapts the recent consistency results of CART and random forests \citep{scornet2015consistency} to XBART with a single tree. More specifically, we show that the single-tree XBART satisfies the sufficient conditions stated in \cite{scornet2015consistency} for $\mathcal{L}^2$ consistency.  The proof in \cite{scornet2015consistency} proceeds by showing that the split criterion satisfies certain sufficient conditions, which are presented in the following section. To adapt this proof to XBART involves two steps. First, note that CART optimizes its split criterion to select a cutpoint, while XBART draws cutpoints at random. We reconcile this difference in the two methods ---  optimizing versus sampling of cutpoints --- by applying the perturb-max lemma, showing the XBART's cutpoint sampling strategy is equivalent to \textit{optimizing} a random objective function, where the randomness vanishes asymptotically. Second, we prove that a key lemma used by \cite{scornet2015consistency} applies to the XBART split criterion as well.

We present the main result first. Theorem \ref{maintheorem} states that a single-tree XBART fit approximates the true underlying mean function in the $\mathcal{L}^2$ norm if the maximum allowed depth goes to infinity slower than a certain function of the sample size. We state the theorem now in terms of a technical condition which will be defined shortly.
\begin{theorem}\label{maintheorem}
	Assume Condition \ref{lemma1} (cf. section \ref{sec:consis}) holds and that $||f||_\infty < \infty$ and $f$ is continuous on $[0,1]^p$. Let $\widehat{f}_n(\x)$ denote a single-tree XBART fit (without variable subsampling or a null cutpoint). Let $n\rightarrow \infty$, $d_n \rightarrow \infty$ and $(2^{d_n}-1)(\log n)^9 / n \rightarrow 0$.  Then XBART is consistent in the sense that
	\begin{equation*}
		\lim_{n\rightarrow \infty} \mathbb{E}[\widehat{f}_n(\x) - f(\x)]^2 = 0,
	\end{equation*}
	where the expectation $\mathbb{E}$ is over $\x$, which is uniformly distributed over $[0,1]^p$.
\end{theorem}
Before proceeding, Section \ref{sec:notation} reviews necessary notations. Section \ref{equivalence} invokes the perturb-max theorem to show that the sampling cutpoint strategy of XBART is equivalent to optimizing a stochastic split criterion. Section \ref{sec:consis} presents sufficient conditions (Condition \ref{lemma1} and Lemma \ref{lemma2}) the split criterion must satisfy to guarantee tree consistency. The proof that Condition \ref{lemma1} and Lemma \ref{lemma2} imply Theorem \ref{maintheorem} may be found in \cite{scornet2015consistency}, which in turn appeals to Theorem 10.2 of \cite{gyorfi2006distribution}.

\subsubsection{Notation}\label{sec:notation}
Without loss of generality, let $\x = (\x^{(1)}, \cdots, \x^{(p)})^t$ be uniformly distributed over $[0,1]^p$ and $y \in \mathbb{R}$. The estimand is $f(\x) = \mathbb{E}\left[Y\mid \x\right]$ and the data are $\mathcal{D}_n = \{(y_1, \x_1), \cdots, (y_n, \x_n)\}$. Denote the fitted XBART single tree by $\widehat{f}_n(\x; \mathcal{D}_n): [0,1]^p \rightarrow \mathbb{R}$. Let $d_n$ denote maximum allowed depth of the tree. In the following sections, we will write $\widehat{f}_n(\x)$ rather than $\widehat{f}_n(\x ; \mathcal{D}_n)$ to lighten notation. As before, the $n\times p$ matrix $\X = (\x_{1}, \cdots, \x_{n})^t$ indicates all covariate data with $n$ observations and $p$ variables, and $\x_i = \left(\x_{i}^{(1)}, \cdots, \x_{i}^{(p)}\right)^t$ is the covariate vector of the $i$-th observation.

%
%

\subsubsection{Connection of XBART sampling to CART optimization}\label{equivalence}
Both single-tree XBART and CART grow regression trees recursively, but use different procedures to select cutpoints: CART optimizes its split criterion while XBART draws cutpoints randomly with probability proportional to its (distinct) split criterion. However, these two cutpoint selection methods are not as different as they seem, due to a well-known result called the {\em perturb-max lemma}, which shows that random sampling is equivalent to optimizing a stochastic  objective function. The perturb-max lemma shows that XBART's approach of sampling the cutpoints is equivalent to an optimization problem, making it possible to verify that the sufficient conditions of consistency in \cite{scornet2015consistency} apply to XBART. We state the perturb-max lemma for reference (following the presentation in \cite{hazan2016perturbations}, Corollary 6.2).
\begin{lemma}[Perturb-max, \cite{hazan2016perturbations}, Corollary 6.2]\label{perturb_max}
	Suppose there are $|\mathcal{C}|$ finite cutpoint candidates $\{c_{jk}\}$ at a specific node. Let $l(c_{jk}) = \log(L(c_{jk}))$ denote logarithm of the split criterion in expression (\ref{eqn:split_gaussian}). We are interested in drawing one of them according to probability $\mbox{P}(c_{jk}) = \frac{\exp\left[l({c}_{jk})\right]}{\sum_{c_{jk}\in \mathcal{C}} \exp\left[l({c}_{jk})\right]} $. We have
	\begin{equation}\label{eqn:perturb}
		\frac{\exp\left[l({c}_{jk})\right]}{\sum_{c_{jk} \in \mathcal{C}} \exp\left[l({c}_{jk})\right]} = \mbox{P}\left(c_{jk} = \arg\max_{c_{jk}\in \mathcal{C}} \{l(c_{jk})+ \gamma_{jk} \} \right)
	\end{equation}
	where $\{\gamma_{jk}\}$ are independent random draws from a Gumbel$(0,1)$ distribution with density $p(x) = \exp\left[ - x + \exp( -x )\right]$.
\end{lemma}
Operationally, Lemma \ref{perturb_max} implies that sampling cutpoints according to the probability on the left-hand side of expression \ref{eqn:perturb} can be achieved as follows:
\begin{enumerate}
	\item Calculating $l(c_{jk})$ for all cutpoint candidates $c_{jk} \in \mathcal{C}$. Draw $\gamma_{jk}$ from a Gumbel$(0,1)$ distribution. 
	\item Pick cutpoint $c_{jk}$ that maximizes the objective function $l(c_{jk}) + \gamma_{jk}$.
\end{enumerate}
Next, note that this optimization problem is invariant if the objective function is scaled by a constant $n$,  used here to denote the number of observations in the current node, so that
\begin{equation*}
	\arg\max_{c_{jk}\in \mathcal{C}} \frac{l(c_{jk})}{n} + \frac{\gamma_{jk}}{n},
\end{equation*}
and define the ``empirical'' split criterion as
\begin{equation}\label{eqn:split_criterion}
	L_n(c_{jk}) = \frac{l(c_{jk})}{n} + \frac{\gamma_{jk}}{n}.
\end{equation}
Therefore, CART and single-tree XBART differ only in terms of the objective function (split criterion) they optimize. In the following proof of the main theorem, we verify that the empirical split criterion of XBART in equation (\ref{eqn:split_criterion}) satisfies sufficient conditions for consistency.

Following the terminology of \cite{scornet2015consistency}, we refer to the limit $L^{*}(c_{jk}) = \lim_{n\rightarrow \infty} L_n(c_{jk})$ as the ``theoretical'' split criterion, and $L_n(c_{jk})$ the ``empirical'' split criterion based on data with $n$ observations. Furthermore, we refer to a tree grown according to the empirical or the theoretical split criterion as an ``empirical tree" or a ``theoretical tree", respectively.

\begin{lemma}
	Theoretical split criterion of XBART: By the strong law of large numbers,  $L_n(c_{jk}) \rightarrow L^{*}(c_{jk})$ almost surely as $n\rightarrow \infty$, for all cutpoint $c_{jk}$, where 
	\begin{equation}\label{eqn:theoretical_split}
	\small
	\begin{aligned}
		L^{*}(c_{jk}) &=\frac{1}{\sigma^2}\left[ P(\x^{(j)} \leq c_{jk}) \left( \mathbb{E}(Y \mid \x^{(j)} \leq c_{jk})\right)^2  + P(\x^{(j)} > c_{jk}) \left( \mathbb{E}(Y \mid \x^{(j)} > c_{jk}) \right)^2 \right].
	\end{aligned}
	\end{equation}
\end{lemma}

See Appendix \ref{sec:convergence} for the proof. 

\begin{remark}
It is important for the proof to note that $L^{*}(c_{jk})$ does not rely on the training data. Observe also that the theoretical split criterion of XBART and CART are equivalent up to a linear transformation with constant coefficients\footnote{Specifically, the ``theoretical'' (asymptotic) CART and XBART criteria have the relationship $L^{*}_{\text{CART}}(c_{jk}) =  \left[\mathbb{E}(Y \mid \x \in \mathcal{A})\right]^2  - \sigma^2 L^{*}_{\text{XBART}}(c_{jk}),$ where $\mathcal{A}$ is the current parent node, $\sigma^2$ is the error variance parameter within the XBART model, which is assumed fixed, but need not equal the true error variance. See the supplementary material for a proof.}.
\end{remark}

\subsubsection{Sufficient conditions of consistency}\label{sec:consis}
This section establishes sufficient conditions for a recursively fit regression tree, selecting cutpoints by optimization, to be $L^2$ consistent. The proof of the single-tree XBART algorithm satisfies those conditions are presented in the Supplementary Material.

The intuition behind the proof of consistency of CART in \cite{scornet2015consistency} is to show that variation of the true function over each hyper-rectangular cell (associated to a leaf node) shrinks as the number of data observations grows larger. Specifically, the variation within a cell can become small in one of two ways. Either, one, because the diameter (largest edge) of the cell shrinks to zero (and the function is continuous and has finite infinity norm), or; two, because the true function is constant over any cell of non-shrinking diameter. This intuition can be formalized via two conditions.

Stating these conditions requires additional notation. To facilitate easy reference to \cite{scornet2015consistency}, we follow their notation in this section. Write $c = (c^{(1)}, c^{(2)})$ to represent a cutpoint, where $c^{(1)}\in \{1,\cdots, p\}$ indicates the variable being cut on and $c^{(2)} \in [0,1]$ indicates the value to cut at. A sequence of $k$ cutpoints from the root until depth $k$ is  $\cc_k = (c_1,\cdots, c_k)$, where $c_k = (c^{(1)}_k, c^{(2)}_k)$ is the $k$-th cutpoint. Let $\mathcal{A}_n(\x,\boldsymbol{\Gamma})$ denote the leaf node of an \textit{empirical tree} (grown based on the empirical split criterion) built with random parameter $\boldsymbol{\Gamma}$ that contains $\x$. Let $\mathcal{A}_k^{*}(\x,\boldsymbol{\Gamma})$ be a cell of the \textit{theoretical tree} (grown based on the theoretical split criterion) at depth $k$ containing $\x$. Additionally, $\mathcal{A}(\x, \cc_k)$ is the node containing $\x$ in the tree built with the sequence of cutpoints $\cc_k$, and we call $\mathbb{A}_k(\x)$ the set of all possible $k\geq 1$ cutpoints used to create the node containing $\x$. The distance between two cut sequences $\cc_k, \cc_k'\in \mathbb{A}_k(x)$ is defined as
\begin{equation*}
	||\cc_k - \cc_k'||_{\infty} = \sup_{1\leq j \leq k}\max \left(\left|c_j^{(1)} - c_{j}'^{(1)}\right|, \left|c_j^{(2)} - c_j'^{(2)}\right| \right).
\end{equation*}
The distance between a cut $\cc_k$ and a set $\mathbb{A} \subset \mathbb{A}_k(\x)$ is defined as
\begin{equation*}
	c_\infty(\cc_k, \mathbb{A}) = \inf_{\cc\in \mathcal{A}} ||\cc_k - \cc||_\infty.
\end{equation*}
Define the total variation of $f$ within any leaf node $\mathcal{A}$ as
\begin{equation*}
	\Delta(f, \mathcal{A}) = \sup_{\x, \x' \in \mathcal{A}} |f(\x) - f(\x')|.
\end{equation*}
We may now state the first condition invoked in Theorem \ref{maintheorem}\footnote{The Condition 1 and Lemma 1 in this paper correspond to Lemma 1 and Lemma 2 in \cite{scornet2015consistency} respectively.}. 
\begin{condition}[Vanishing total variation]\label{lemma1}
	For all $\x \in (0, 1)^p$,
	\begin{equation*}
		\Delta(f, \mathcal{A}_k^{*}(\x, \boldsymbol{\Gamma}))\rightarrow 0 \qquad \text{almost surely as } k \rightarrow \infty.
	\end{equation*}
\end{condition}

Condition \ref{lemma1} states that as $n\rightarrow\infty$, variation of the true function $f$ tends to zero in the leaf nodes of a \textit{theoretical} tree. \cite{scornet2015consistency} show that Condition 1 is satisfied by CART if $f$ is additive, the elements of $\x$ are independent, and the errors are Gaussian (their Lemma 1 and assumption 1). Because XBART and CART have the same theoretical split criterion (see the supplementary material for a proof), their proof applies directly to XBART, so Condition 1 is satisfied for us as well if $f$ satisfies those rather strict requirements. Plausibly, Condition 1 is satisfied for a broader class of functions in conjunction with the XBART (equivalently, CART) theoretical tree and for this reason we argue that Condition \ref{lemma1} is better treated as an assumption.

%
%
Now we move on to Lemma \ref{lemma2}, which states that the cutpoints of the theoretical and empirical trees are close to one other in a certain sense. For any $\x \in [0,1]^p$, and cuts $\cc_k$, define the empirical split criterion (as of equation (\ref{eqn:split_criterion})) for XBART as
\begin{equation}\label{eqn:XBART_empirical_split}
\small
\begin{split}
	L_{n,k}(\x, \cc_k)  =& \frac{\tau \mathrm{N}_n(\mathcal{A}_{L})}{\sigma^2 \left(\sigma^2 + \tau \mathrm{N}_n(\mathcal{A}_{L}) \right)}\frac{1}{n}\left(\sum_{i:\x_{i}^{(c_k^{(1)})} \leq c_k^{(2)}}y_i^2 \;\;\;-\sum_{i:\x_{i}^{(c_k^{(1)})} \leq c_k^{(2)}} (y_i - \bar{y}_l)^2 \right) \\
		        +& \frac{\tau \mathrm{N}_n(\mathcal{A}_{R})}{\sigma^2 \left(\sigma^2 + \tau \mathrm{N}_n(\mathcal{A}_{R}) \right)}\frac{1}{n}\left(\sum_{i:\x_{i}^{(c_k^{(1)})} > c_k^{(2)}}y_i^2\;\;\; -  \sum_{i:\x_{i}^{(c_k^{(1)})} > c_k^{(2)}}  (y_i - \bar{y}_r)^2 \right)  \\
		        +&  \frac{\gamma_x}{n}.
\end{split}
\end{equation}
where $\mathrm{N}_n(\mathcal{A}_{L})$ and $\mathrm{N}_n(\mathcal{A}_{R})$ denote the number of observations in the node $\mathcal{A}(\x, \cc_{k-1}) \cap \{\mathbf{z}: z^{(c_k^{(1)})} \leq c_k^{(2)}\}$
and  $\mathcal{A}(\x, \cc_{k-1}) \cap \{\mathbf{z}: z^{(c_k^{(1)})} > c_k^{(2)}\}$
respectively and $\bar{y}_l$ and $\bar{y}_r$ denote the left and right observation means. The function $L_{n,k}  (\x, \cc_k) $ is the empirical split criterion for the node $\mathcal{A}(\x, \cc_{k-1})$ expressed in terms of previous cuts $\cc_{k-1}$ and the current cut $c_k$. For all $\xi > 0$ and $\x\in [0,1]^p$, $\mathbb{A}_{k-1}^\xi(\x) \subset \mathbb{A}_{k-1}(\x)$ denotes the set of all sequences of cuts $\cc_{k-1}$ such that the node $\mathcal{A}(\x, \cc_{k-1})$ contains a hypercube with edge length $\xi$. The set $\bar{\mathbb{A}}_k^\xi (\x) = \{\cc_k : \cc_{k-1} \in \mathbb{A}_{k-1}^\xi(\x)\}$ is equipped with norm $||\cdot||_\infty$. We may now state the key lemma required in  \cite{scornet2015consistency} to prove Theorem 1.
\begin{lemma}[Stochastic equicontinuity]\label{lemma2}
	Assume that  $||f||_\infty < \infty$ and $f$ is continuous on $[0,1]^p$. Fix $\x \in [0,1]^p$, $k \in \mathbb{N}^{*}$ and let $\xi > 0$. Then $L_{n,k}(\x,\dot)$ is stochastically equicontinuous on $\bar{\mathbb{A}}_k^\xi (\x)$:  for all $\alpha$, $\rho >0$, there exist $\delta > 0$ such that
	\begin{equation*}
		\lim_{n\rightarrow \infty} \mathbb{P}\left[\sup_{\substack{||\cc_k - \cc_k'||_\infty \leq \delta \\\cc_k, \cc_k'\in \bar{\mathbb{A}}_k^\xi (\x)}} \left| L_{n,k}(\x, \cc_k) - L_{n,k}(\x, \cc_k')\right| > \alpha \right] \leq \rho.
	\end{equation*}
\end{lemma}

\noindent A proof that Lemma \ref{lemma2} holds for XBART split criterion in equation (\ref{eqn:XBART_empirical_split}) is presented in the Supplementary Material. Our proof strategy is the same as \cite{scornet2015consistency}, but had to be verified using new bounding arguments specific to the XBART criterion.

\subsection{Stationarity of the forest algorithm}\label{stationarity}
This section proves that a slightly modified version of {\sc GrowFromRoot} generates draws from a Markov chain with a stationary distribution.  The slight modification is that all leaf parameters are drawn jointly, conditional on the current state of the forest, prior to sampling (growing) each new tree.

\begin{theorem}\label{theorem:stationary}
	{\sc XBART} samples $\mathcal{F} = \{T_h\}_{1\leq h \leq L}$ according to a finite-state Markov chain with stationary distribution.
\end{theorem}

\begin{proof}
The proof proceeds by showing that the forest sampling process 1) is a Markov chain, 2) has only finite states, and 3) that the transition probability between any two states is positive. Therefore, by standard results (for example Theorem 1.7 and 1.20 of \cite{durrett2016essentials}), it has a stationary distribution.

	Let $\mathcal{F}= \{T_h\}_{1\leq h \leq L}$ denote the forest of $L$ trees and let $\boldsymbol{\mu} = \{\boldsymbol{\mu}_h\}_{1\leq h \leq L}$ denote the associated leaf parameters. Let $\mathcal{F}_{-j} = \{T_{h}\}_{1\leq h \leq L} / T_j$ and $\boldsymbol{\mu}_{-j} = \{\boldsymbol{\mu}_h\}_{1 \leq h \leq L}/ \boldsymbol{\mu}_j$ be sets of trees and leaf parameters excepting the $j$th one. 
	
\begin{enumerate}	

\item {\sc GrowFromRoot} explicitly updates $T_j$ given $\mathcal{F}_{-j}$ and therefore defines a Markov chain.	More explicitly, $\mathcal{F}^{(k)}$ is drawn by sampling and replacing tree $T_j$, conditional on $\mathcal{F}_{-h}^{(k-1)}$.

\item {\sc GrowFromRoot} samples trees from a finite state space. Each tree has a maximum depth and all cutpoint candidates are defined in terms of a finite predictor matrix $\X$, so the total number of tree configurations is finite. The forest $\mathcal{F}$ is an ensemble of a finite number of trees, thus has a finite number of states as well.

\item The probability that {\sc GrowFromRoot} draws a given tree is a product of the non-zero probabilities of drawing specific cutpoints at each node; therefore, the probability of drawing any specific tree is non-zero. 
	
	Specifically, 
	\begin{equation}\label{eqn:tree_transit}
	\small
		p\left(T_j^{(k)} \bigm | \mathcal{F}^{(k-1)}\right ) =\int p\left(T^{(k)}_j \bigm | \y, \mathcal{F}^{(k-1)}_{-j}, \boldsymbol{\mu}_{-j} \right) \pi\left(\boldsymbol{\mu}_{-j}, \sigma^2 \bigm | \y, \mathcal{F}^{(k-1)}\right)  d(\boldsymbol{\mu}_{-j}, \sigma^2) > 0,
	\end{equation}
for any $T_j$. The second factor of the integrand, $ \pi\left(\boldsymbol{\mu}_{-j}, \sigma^2 \bigm | \y, \mathcal{F}^{(k-1)}\right)$, denotes a draw from a conjugate linear regression with design matrix given by dummy variables indicating leaf membership (and simply discarding the parameters associated with the $j$th tree). The first factor in the integrand, $p\left(T^{(k)}_j \bigm | \y, \mathcal{F}^{(k-1)}_{-j}, \boldsymbol{\mu}_{-j} \right)$ denotes the product of {\sc GrowFromRoot} probabilities leading to a draw of $T_j$. Both factors are always greater than zero.

Finally, observe that there is at least one way to transition from any forest $\mathcal{F}$ to any other forest $\mathcal{F}'$, which is to regrow each tree and replace them one by one over exactly $L$ iterations. 	
\end{enumerate}	
\end{proof}


\section{Discussion}\label{discussion}

To conclude, we briefly describe the chronological development of the XBART algorithm, to provide additional context for evaluating the relative merits of XBART, BART (fit with traditional MCMC), and XGBoost.

The XBART algorithm grew out of our attempts to better understand BART's exceptional empirical performance. In routine use, we found that BART often outperformed XGBoost at function estimation and prediction, sometimes substantially so. Unfortunately, there were some data sets that were simply too large for us to apply BART to, while XGBoost is notoriously fast. Our initial hypothesis was that BART's clever regularization might be behind its remarkable performance, so we set about to create a fast recursive tree-fitting algorithm that utilized a penalty analogous to the BART prior. We found that this approach did not work as well as BART. Next, we conjectured that the BART splitting criterion might itself be the source of BART's advantage and we implemented a version that greedily optimized BART's marginal likelihood criteria when growing the trees. This too, did not match BART's performance, so we decided to try sampling the cutpoints as currently performed in XBART, while optimizing the leaf parameters (rather than sampling). This led to a notable improvement, but still typically under-performed BART in our comparisons. So, we implemented the sampling of the leaf parameters, leading to the XBART algorithm described in this paper. In addition, we notice that sampling the prior variance $\tau$ improves the performance on high noise and few data observations cases significantly. Finally, at this point, our new algorithm mimicked BART's performance and sometimes even outperformed it --- especially in larger problems where BART could not be run long enough to achieve adequate mixing.

To put this journey into perspective, our initial goal was to create a CART-like recursive optimization algorithm and introduce certain elements of BART. But, one by one as we incorporated these elements, each one yielded additional performance gains. In the end, we ended up with a BART fitting algorithm with a novel recursive growing scheme --- mostly BART with a dash of CART, rather than vice-versa as we had initially planned.

The end result of these experiments and the accompanying algorithm development was the completely novel function estimation method described in Section \ref{sec:gaussian}. The remainder of the paper represents our attempts to better understand the operating characteristics of the new algorithm. Our key findings were that it has excellent performance on simulated and empirical data across various signal-to-noise regimes and test functions (Section \ref{sec:example}). We also confirmed that it provides superior estimation and inference when used to initialize the BART MCMC algorithm, compared to the standard initialization scheme (Section \ref{sec:warm_start}). Finally, we were able to establish several basic theoretical facts about the algorithm (Section \ref{sec:theory}). Though these theoretical results are limited, they also address many of the prima facie objections one might have regarding a pseudo-Bayesian sampling algorithm, namely consistency and stationarity. Extending the single-tree results presented here to the ensemble version of the model is the subject of ongoing research.

BART has proven to be a widely used model in many fields, combining state-of-the-art estimation and prediction with fully Bayesian uncertainty quantification. The XBART algorithm permits these virtues to be realized on large data sets that were previously out of reach for existing implementations. The generality of the {\sc GrowFromRoot} algorithm suggests that the recursive stochastic search strategy at the heart of XBART can be readily adapted to other models and could yield similar accuracy and computational efficiency improvements as those seen in regression problems.

The software package \texttt{XBART} is available online at \url{http://www.github.com/jingyuhe/xbart} for both \texttt{R} and \texttt{python}. The package is still undergoing active development for further extensions and performance optimizations.
\bibliographystyle{chicago}
\bibliography{BART}

\newpage
\appendix

\section{Computational considerations}\label{sec:computation}
This section catalogs implementation details that improve the computational efficiency of the algorithm. These implementational details serve to make the algorithm competitive with state-of-the-art supervised learning algorithms, such as XGBoost. These particular strategies, such as variable presorting and careful handling of categorical covariates, are inapplicable in the standard BART MCMC and XBART's ability to incorporate them is the basis of its improved performance.

\subsection{Adaptive variable importance weights}
Our XBART implementation strikes an intermediate balance between the local BART updates, which randomly consider one variable at a time, and the all-variables Bayes rule described above. Specifically, we consider only $m \leq V$ variables at a time when sampling each cutpoint. Rather than drawing these variables uniformly at random as is done in random forests, we introduce a parameter vector $\w$, which denotes the prior probability that a given variable is chosen to be split on, as suggested in \cite{linero2018bayesian2}. Before sampling each cutpoint, we randomly select $m$ variables (without replacement) with probability proportional to $\w$.

\subsection{Pre-sorting predictor variables}\label{pre_sort}
Observe that the XBART split criterion depends on sufficient statistics only, namely the sum of the observations in a node (that is, at a given level of the recursion). An important implication of this, for computation, is that with sorted predictor variables, the various cutpoint integrated likelihoods can be computed rapidly via a single sweep through the data (per variable), taking cumulative sums. Let $\mathbf{O}$ denote the $V$-by-$n$ array such that $o_{vh}$ denotes the index, in the data, of the observation with the $h$-th smallest value of the $v$-th predictor variable $x_v$. Then, taking the cumulative sums gives
\begin{equation*}
	s(\leq,v,c) = \sum_{h \leq c} \res_{o_{vh}}
\end{equation*}
and
\begin{equation*}
	s(>,v,c) = \sum_{h = 1}^n r_{lh}-  s(\leq,v,c).
\end{equation*}
The subscript $l$ on the residual indicates that these evaluations pertain to the update of the $l$-th tree.

The above formulation is useful if the data can be presorted and, furthermore, the sorting can be maintained at all levels of the recursive tree-growing process. To achieve this, we must loop over (sift) each of the variables before passing to the next level of the recursion. Specifically, we form two new index matrices $\mathbf{O}^{\leq}$ and $\mathbf{O}^{>}$ that partition the data according to the selected cutpoint. For the selected split variable $v$ and selected split $c$, this is automatic: $O_v^{\leq} = O_{v,1:c}$ and $O_v^{>} = O_{v,(c+1):n}$. For the other $V-1$ variables, we sift them by looping through all $n$ available observations, populating $O^{
			\leq}_{q}$ and $O^{>}_{q}$, for $q \neq v$, sequentially, with values $o_{qj}$ according to whether $x_{vo_{qj}} \leq c$ or $x_{vo_{qj}} > c$, for $j = 1, \dots, n$.

Because the data is processed in sorted order, the ordering will be preserved in each of the new matrices $\mathbf{O}^{\leq}$ and $\mathbf{O}^{>}$. This strategy was first presented in \cite{mehta1996sliq} in the context of classification algorithms and has been rediscovered a number of times since then. The pre-sorting and sifting $\mathbf{O}$ strategy is easy to implement for continuous covariates, but not for categorical covariates due to the possibility of ties in the data. Appendix \ref{categorical_covariates} describes a special data structure for dealing with ties efficiently.

\subsection{Adaptive cutpoint grid}
Evaluating the integrated likelihood criterion is straightforward, but the summation and normalization required to sample the cutpoints contribute a substantial computational burden itself. Therefore, it is helpful to consider a restricted number of cutpoints $C$. 

This can be achieved simply by taking every $j$th value (starting from the smallest) as an eligible cutpoint with $j = \lfloor \frac{n_b-2}{C} \rfloor$. As the tree grows deeper, the amount of data that is skipped over diminishes. Eventually, we get $n_b < C$, and each data point defines a unique cutpoint. In this way, the data could, without regularization, be fit perfectly, even though the number of cutpoints at any given level is given an upper limit. As a default, we set the number of cutpoints to $\min{(n,100)}$, where $n$ is the sample size of the entire data set.

Our cutpoint subsampling strategy is more straightforward than the elaborate cutpoint subselection search heuristics used by XGBoost \citep{chen2016xgboost} and  LightGBM \citep{ke2017lightgbm}, which both consider the gradient evaluated at each cutpoint when determining the next split. Our approach does not consider the response information at all but rather defines a predictor-dependent prior on the response surface. That is, given a design matrix $\mathbf{X}$, sample functions can be drawn from the prior distribution by sampling trees, splitting uniformly at random among the cutpoints defined by the node-specific quantiles, in a sequential fashion.

\subsection{Variable importance weights}

XBART can strike a balance between local BART updates, which randomly consider one variable at a time, and the all-variables Bayes rule described above by only considering $m \leq V$ variables when evaluating the cutpoints. Rather than drawing these $m$ variables uniformly at random, as is done in random forests, we introduce a parameter vector $\w$, which denotes the prior probability that a given variable is chosen to be split on, as suggested in \cite{linero2018bayesian2}. Before sampling each cutpoint, we randomly select $m$ variables (without replacement) with probability proportional to $\w$.

The variable weight parameter $\w$ is given a Dirichlet prior with hyper-parameters $\bar{\w}$ that is initialized to all ones. At each iteration of the first sweep through the forest, $\bar{\w}$ is incremented to count the total number of splits across all trees. The split counts are then updated in between each tree sampling/growth step:
\begin{equation*}
	\bar{\w} \leftarrow \bar{\w} - \bar{\w}_l^{(k-1)} +\bar{\w}_l^{(k)}
\end{equation*}
where $\bar{\w}_l^{(k)}$ denotes the length-$V$ vector recording the number of splits on each variable in tree $l$ at iteration $k$. The weight parameter is then re-sampled as
$\w \sim \mbox{Dirichlet}(\bar{\w}).$
Splits that improve the likelihood function will be chosen more often than those that do not. The parameter $\w$ is then updated to reflect that, making chosen variables more likely to be considered in subsequent sweeps. In practice, we find it is helpful to use all $V$ variables during an initialization phase, to more rapidly obtain an accurate initial estimate of $\w$.

\subsection{Categorical covariates}\label{categorical_covariates}
Section \ref{pre_sort} suggests pre-sorting covariates to compute sufficient statistics efficiently, this strategy is straightforward for continuous covariates. However, because of possible ties in ordered categorical covariates, a more efficient algorithm is needed to calculate sufficient statistics.

We restate notations in section \ref{pre_sort}. Without loss of generality, we assume that all covariates are categorical. Let $\mathbf{O}$ denote the $V$-by-$n$ array such that $o_{vh}$ denotes the index, in the data, of the observation with the $h$-th smallest value of the $v$-th predictor variable $x_v$. Then, taking the cumulative sums gives
\begin{equation*}
	s(\leq,v,c) = \sum_{h \leq c} \res_{o_{vh}}
\end{equation*}
and
\begin{equation*}
	s(>,v,c) = \sum_{h = 1}^n r_{lh}-  s(\leq,v,c).
\end{equation*}

\begin{algorithm*}[h!]
	\begin{algorithmic}[1]
		\caption{Pseudocode of calculating sufficient statistics for categorical covariates.}\label{alg:categorical}
		\State Sort categorical covariates, create $\mathbf{O}$ matrix. Count number of unique observations \texttt{unique\_val} and \texttt{val\_count} vector (suppose vectors are length $K$).
		\For{i from 1 to K}
		\State Calculate sufficient statistics for cutpoint candidate \texttt{unique\_val}[i] as
		\begin{equation*}
			s(\leq, v, \texttt{unique\_val}[i]) = \sum_{h \in \left[\sum_{m = 1}^{i - 1}\texttt{val\_count[m]}, \sum_{m = 1}^{i}\texttt{val\_count[m]}\right]} r_{o_{vh}}.
		\end{equation*} and
		\begin{equation*}
			s(>,v,c) = \sum_{h = 1}^n r_{lh}-  s(\leq,v,c).
		\end{equation*}
		\EndFor
		\State Calculate split criterion, determine a cutpoint.
		\If{stop-splitting is selected or stop conditions are reached}
		\State Draw leaf parameters and \textbf{return}.
		\Else
		\State Sift \texttt{unique\_val} and \texttt{val\_count} for left and right child nodes. Repeat step 3 when evaluate split criterion at child nodes.
		\EndIf
	\end{algorithmic}
\end{algorithm*}

The subscript $l$ on the residual indicates that these evaluations pertain to the update of the $l$th tree. Notice that when covariates are categorical, $x_{vh}$ is not necessarily smaller than $x_{v(h+1)}$ due to potential ties in $x$. As a result, the number of unique cutpoint candidates is less than $n$. We propose an extra data structure to bookkeeping unique cutpoint and number of ties as follows. For the $v$-th categorical predictor variable $x_v$, let \texttt{unique\_val} be a vector of unique values (sorted, from small to large) in $x_v$ and \texttt{val\_count} be a vector of counts of replication for each unique value. Therefore, the cutpoint candidate is a element in the vector \texttt{unique\_val}, say the $i$-th element. Then the cumulative sums is
\begin{equation*}
	s\left(\leq, v, \texttt{unique\_val}[i]\right) = \sum_{h \in \left[\sum_{m = 1}^{i - 1}\texttt{val\_count[m]}, \sum_{m = 1}^{i}\texttt{val\_count[m]}\right] } r_{o_{vh}}.
\end{equation*}
When sifting data to left and right child after drawing a cutpoint, we create the same \texttt{unique\_val} and \texttt{val\_count} vector for all categorical covariates with data in two child nodes respectively. See Algorithm \ref{alg:categorical} for details.

\section{Demo of XBART forest algorithm}\label{sec:demo:XBART}
Figure \ref{fig:demo} depicts the fitting process of a simple XBART forest with three trees and two sweeps for Gaussian regression. We label the fitting target of each tree in each sweep explicitly.
\begin{figure}[!htbp]
	\centering
	\includegraphics[width = 0.71\textwidth]{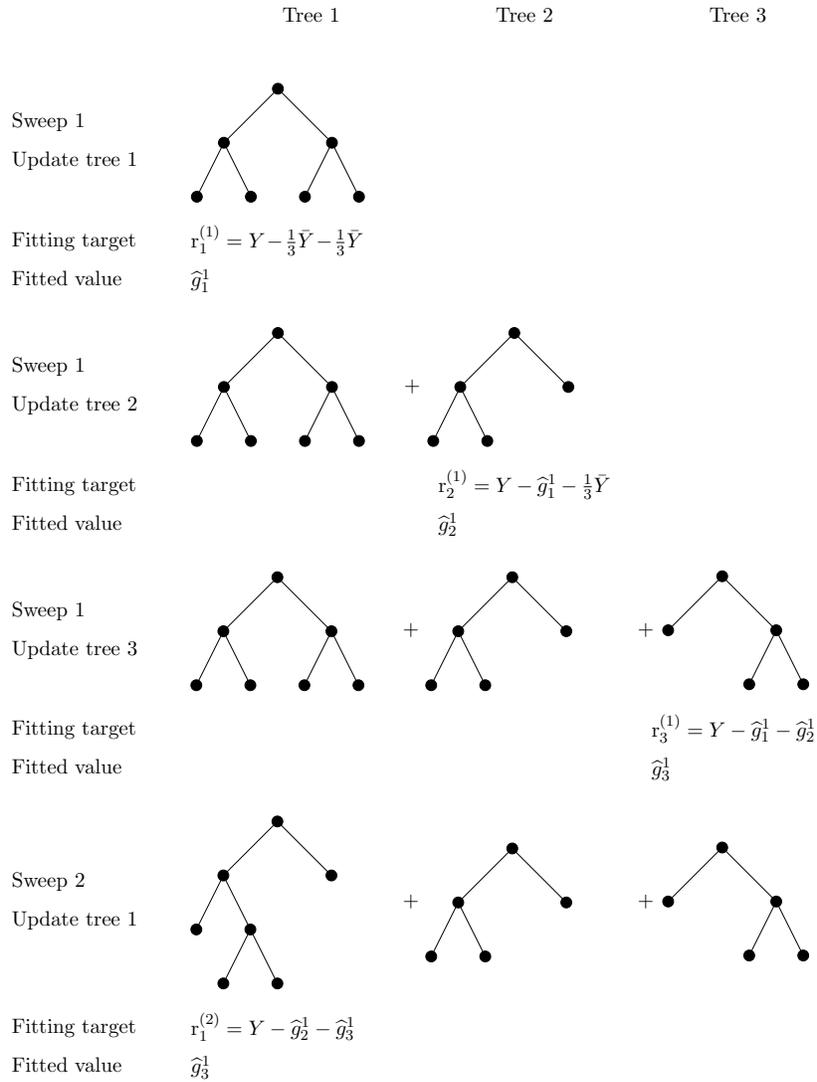}
	\caption{A simple demonstration of the XBART forest fitting procedure, where $\bar{Y}$ denote average of all $Y$ observations. It shows how to initialize the fitted values for the three trees in the first sweep, and update the first tree in sweep 2. The following sweeps proceed similarly. }\label{fig:demo}
\end{figure}

\section{Convergence of the empirical split criterion (Lemma 2)}\label{sec:convergence}
We show that the empirical split criterion in equation (\ref{eqn:split_criterion}) converges to the theoretical one (\ref{eqn:theoretical_split}). Note that $\gamma_{jk}$ is a draw from Gumbel(0,1) distribution.
\begin{equation}\label{eqn:appendix:criterion}
\begin{aligned}
	L_n(c_{jk}) &= \frac{l(c_{jk})}{n} + \frac{\gamma_{jk}}{n} \\
	&= \frac{1}{n}\frac{\tau n_{jk}^l}{\sigma^2 \left(\sigma^2 + \tau n_{jk}^l \right)}\left(\sum_{i:\x_i \in \mathcal{A}_L(j,k)}y_i^2 - \sum_{i:\x_i \in \mathcal{A}_L(j,k)} (y_i - \bar{y}_l)^2 \right)                 \\
		          & \qquad + \frac{1}{n}\frac{\tau n_{jk}^r}{\sigma^2 \left(\sigma^2 + \tau n_{jk}^r \right)}\left( \sum_{i:\x_i \in \mathcal{A}_R(j,k)} y_i^2 - \sum_{i:\x_i \in \mathcal{A}_R(j,k)} (y_i - \bar{y}_r)^2 \right) \\
		          & \qquad\qquad +  \frac{1}{n}\log{ \left( \frac{\sigma^2}{\sigma^2 + \tau n_{jk}^l} \right)} +  \frac{1}{n}\log{ \left( \frac{\sigma^2}{\sigma^2 + \tau n_{jk}^r} \right)} + \frac{1}{n}\gamma_{jk}.
		          \end{aligned}
\end{equation}
As $n\rightarrow \infty$, since $n_{jk}^l < n$, it is straight forward that the last three terms converge to zero,
\begin{equation*}
\frac{1}{n}\log{ \left( \frac{\sigma^2}{\sigma^2 + \tau n_{jk}^l} \right)} \rightarrow 0, \quad \frac{1}{n}\log{ \left( \frac{\sigma^2}{\sigma^2 + \tau n_{jk}^r} \right)} \rightarrow 0, \quad  \frac{1}{n}{\gamma_{jk}} \rightarrow 0.
\end{equation*}
Furthermore, notice that $\frac{n_{jk}^l}{n}$ converges to $P(\x^{(j)} \leq c_{jk})$, the probability that a new observation falls in the left child, and similarly for $\frac{n_{jk}^r}{n}$. The first term of equation (\ref{eqn:appendix:criterion}) converges to
\begin{equation*}
\begin{aligned}
&\frac{1}{n}\frac{\tau n_{jk}^l}{\sigma^2 \left(\sigma^2 + \tau n_{jk}^l \right)}\left(\sum_{i:\x_i \in \mathcal{A}_L(j,k)}y_i^2 - \sum_{i:\x_i \in \mathcal{A}_L(j,k)} (y_i - \bar{y}_l)^2 \right)   \\
=& \frac{\tau n_{jk}^l}{\sigma^2 \left(\sigma^2 + \tau n_{jk}^l \right)}  \frac{n_{jk}^l}{n} \left(\frac{1}{n_{jk}^l}\sum_{i:\x_i \in \mathcal{A}_L(j,k)}y_i^2 - \frac{1}{n_{jk}^l} \sum_{i:\x_i \in \mathcal{A}_L(j,k)} (y_i - \bar{y}_l)^2 \right)   \\
\rightarrow&\frac{1}{\sigma^2} P(\x^{(j)} \leq c_{jk}) \left[ \mathbb{E}(y^2 \mid \x^{(j)} \leq c_{jk}) - \mathbb{V}(y \mid \x^{(j)} \leq c_{jk}) \right] \\
=& \frac{1}{\sigma^2}P(\x^{(j)} \leq c_{jk})  \left(\mathbb{E}(y \mid \x^{(j)} \leq c_{jk})\right)^2
\end{aligned}
\end{equation*}
Similary, the second term of equation (\ref{eqn:appendix:criterion}) converges to 
\begin{equation*}
\begin{aligned}
& \frac{1}{n}\frac{\tau n_{jk}^r}{\sigma^2 \left(\sigma^2 + \tau n_{jk}^r \right)}\left( \sum_{i:\x_i \in \mathcal{A}_R(j,k)} y_i^2 -  \sum_{i:\x_i \in \mathcal{A}_R(j,k)} (y_i - \bar{y}_r)^2 \right)\\
 \rightarrow&\frac{1}{\sigma^2}P(\x^{(j)} > c_{jk})  \left(\mathbb{E}(y \mid \x^{(j)} > c_{jk})\right)^2
\end{aligned}
\end{equation*}
Taken together, as $n\rightarrow \infty$, 
\begin{equation*}
\begin{aligned}
L_n(c_{jk})  &\rightarrow \frac{1}{\sigma^2}\left[ P(\x^{(j)} \leq c_{jk}) \left( \mathbb{E}(y \mid \x^{(j)} \leq c_{jk})\right)^2  + P(\x^{(j)} > c_{jk}) \left( \mathbb{E}(y \mid \x^{(j)} > c_{jk}) \right)^2 \right].
\end{aligned}
\end{equation*}

\section{Complete simulation results}\label{sec:appendix:simulation}

We compare to leading machine learning algorithms: random forests, gradient boosting machines, neural networks. All implementations had an {\tt R} interface and were the current fastest implementations to our knowledge: {\tt ranger} \citep{wright2015ranger}, {\tt xgboost} \citep{chen2016xgboost}, and {\tt Keras} \citep{chollet2015keras} respectively. For {\tt Keras} we used a single architecture but varied the number of training epochs depending on the noise level of the problem. For {\tt xgboost} we consider two specifications, one using the software defaults and another determined by a 5-fold cross-validated grid optimization (see Table \ref{tab:f7}); a reduced grid of parameter values was used at sample sizes $n > 10,000$. 

\begin{table}[!htbp]
\small
	\centering
	\begin{tabular}{l|rr}
		\toprule
		Parameter name           & $N = 10$K                   & $N > 10$K                 \\
		\hline
		{\tt eta}                & $\lbrace 0.1, 0.3\rbrace$   & $\lbrace 0.1, 0.3\rbrace$ \\

		{\tt max\_depth}         & $\lbrace 4, 8, 12\rbrace$   & $\lbrace 4, 12\rbrace$    \\
		{\tt colsample\_bytree}  & $\lbrace 0.7, 1 \rbrace$    & $\lbrace 0.7, 1 \rbrace$  \\
		{\tt min\_child\_weight} & $\lbrace 1, 10, 15 \rbrace$ & $10$                      \\
		{\tt subsample}          & 0.8                         & 0.8                       \\
		{\tt gamma }             & 0.1                         & 0.1                       \\
		\bottomrule
	\end{tabular}%
	\caption{ Hyperparameter Grid for \texttt{xgboost}.}\label{tab:f7}
\end{table}%

The software used is \texttt{R} version 4.0.3 with  {\tt xgboost} 0.71.2, \texttt{dbarts} version 0.9.1, \texttt{ranger} 0.11.1 and \texttt{keras} 2.3.0.0. The default choice of hyperparameters for {\tt xgboost}  are \texttt{eta} $= 0.3$, {\tt colsample\_bytree} $=1$, {\tt min\_child\_weight} $= 1$ and \texttt{max\_depth} $= 6$. Ranger was fit with \texttt{num.trees} $=500$ and \texttt{mtry} $ = \text{floor}\left(\sqrt{p}\right)$. For {\tt Keras} we build a network with two fully connected hidden layers (15 nodes each) using ReLU activation function, $L_1$ regularization at 0.01, and with 50/20 epochs depending on the signal to noise ratio. The simulation was ran on a cluster with two Intel Xeon Gold 6230 CPU and 320GB memory.

The tables below demonstrate simulation results of XBART, XGBoost with cross validation (XGBCV), XGBoost default parameters (XGB), random forests (RF) and neural networks (NN), on several different settings, including independent regressors (Table \ref{table:ind}), correlated regressors (Table \ref{table:factor}) and fat-tail error (Table \ref{table:t}).

Note that we consider XBART in two cases, fixed $\tau = \text{Var}(\y) / L$, or assign an inverse-Gamma$(3, 0.5 \times \text{Var}(\y) / L)$ and update $\tau$ in between of sweeps. In general, sampling $\tau$ improves the performance dramtically on high noise cases $\kappa = 10$, but the RMSE is slightly higher for the huge $n$ cases such as $n = 250k, \kappa = 1$.

\newpage
\begin{table}[!htbp]
  \centering
  \resizebox{\columnwidth}{!}{%
    \begin{tabular}{rc|cccccc|cccccc}
      \toprule
      \multicolumn{3}{c}{} & \multicolumn{5}{c}{$\kappa = 1$} & \multicolumn{5}{c}{$\kappa = 10$}                                                                                                                                          \\
      \midrule
      \multirow{2}{*}{$p$} & \multirow{2}{*}{$n$} & XBART        & XBART        & \multirow{2}{*}{XGBCV} & \multirow{2}{*}{XGB} & \multirow{2}{*}{RF} & \multirow{2}{*}{NN} & XBART        & XBART        & \multirow{2}{*}{XGBCV} & \multirow{2}{*}{XGB} & \multirow{2}{*}{RF} & \multirow{2}{*}{NN} \\
                           &                      & fixed $\tau$ & sampling $\tau$  &                        &                      &                     &                     & fixed $\tau$ & sampling $\tau$  &                        &                      &                     &                     \\
      \midrule
      \multicolumn{14}{c}{Linear}\\
      \midrule
      500                  & 300                  & 25.94 (3.7)  & 25.38 (3.8)  & 27.16 (3.8)            & 28.53 (0.1)          & 24.87 (0.1)         & 24.14 (2.5)         & 53.42 (3.6)  & 33.7 (3.3)   & 78.2 (3.8)             & 112.24 (0.1)         & 33.08 (0.1)         & 26.57 (1.3)         \\
      100                  & 1k                   & 10.12 (2.1)  & 10.11 (1.9)  & 9.97 (7)               & 10.70 (0.1)           & 10.37 (0.2)         & 12.41 (27)          & 18.55 (1.6)  & 14.27 (1.5)  & 26.56 (7.1)            & 50.71 (0.1)          & 16.03 (0.2)         & 67.26 (11.0)          \\
      1k                   & 500                  & 35.57 (8.5)  & 35.39 (8.7)  & 36.07 (11)             & 40.08 (0.3)          & 35.04 (0.2)         & 34.23 (40.3)        & 67.68 (8.4)  & 45.87 (8.4)  & 103.55 (11.2)          & 148.6 (0.3)          & 49.95 (0.3)         & 172.04 (16.7)       \\
      1k                   & 1k                   & 35.81 (11.3) & 35.5 (13.5)  & 36.27 (21.3)           & 39.86 (0.6)          & 35.21 (0.4)         & 31.93 (28)          & 59.15 (10.6) & 44.88 (12.1) & 85.86 (21.5)           & 161.2 (0.6)          & 43.31 (0.4)         & 208.87 (11.7)       \\
      30                   & 10k                  & 2.14 (8.1)   & 2.12 (5.5)   & 3.09 (28.0)            & 3.23 (0.3)           & 3.63 (1.0)          & 1.32 (27.3)         & 4.98 (3.4)   & 4.71 (4.0)     & 6.26 (25.8)            & 15.46 (0.2)          & 5.99 (1.1)          & 7.12 (11.2)         \\
      30                   & 50k                  & 1.46 (58.4)  & 1.43 (45.3)  & 2.93 (52.7)            & 2.69 (1.2)           & 3.31 (7.3)          & 0.66 (28.9)         & 3.90 (18.2)   & 3.60 (16.5)   & 4.48 (51.1)            & 9.22 (1.2)           & 4.77 (9.2)          & 3.53 (12.3)         \\
      30                   & 250k                 & 0.89 (639.5) & 0.91 (406.1) & 2.75 (540.1)           & 2.45 (11)            & 3.03 (59.7)         & 0.29 (37.1)         & 2.91 (222.7) & 2.72 (153.0)   & 3.91 (504.1)           & 5.40 (10.9)           & 4.11 (83.8)         & 1.87 (15.8)         \\
      \midrule
      \multicolumn{14}{c}{Max}\\
      \midrule
      500                  & 300                  & 1.46 (3.8)   & 1.46 (3.3)   & 1.36 (3.8)             & 1.46 (0.1)           & 1.84 (0.1)          & 2.5 (2.6)           & 4.43 (3.6)   & 2.55 (3.5)   & 7.12 (3.9)             & 10.63 (0.1)          & 2.74 (0.1)          & 5.62 (1.4)          \\
      100                  & 1k                   & 0.92 (1.8)   & 0.91 (1.5)   & 0.92 (7.5)             & 1.2 (0.1)            & 1.43 (0.2)          & 2.67 (27.0)           & 3.62 (1.6)   & 2.57 (1.6)   & 5.06 (7.2)             & 9.88 (0.2)           & 2.98 (0.2)          & 16.48 (11.1)        \\
      1k                   & 500                  & 1.15 (8.8)   & 1.17 (8.5)   & 1.10 (11.4)             & 1.43 (0.3)           & 1.97 (0.3)          & 2.25 (40.6)         & 3.88 (8.4)   & 2.64 (8.5)   & 6.01 (10.8)            & 9.38 (0.3)           & 2.54 (0.2)          & 10.85 (16.5)        \\
      1k                   & 1k                   & 0.91 (11.6)  & 0.94 (15.4)  & 0.97 (21.7)            & 1.28 (0.6)           & 1.99 (0.5)          & 2.33 (27.9)         & 3.48 (10.6)  & 2.34 (12.4)  & 5.28 (21.4)            & 9.72 (0.6)           & 2.67 (0.4)          & 12.94 (11.5)        \\
      30                   & 10k                  & 0.40 (3.9)    & 0.40 (2.3)    & 0.44 (26.5)            & 0.62 (0.2)           & 0.61 (1.1)          & 0.42 (28.4)         & 1.60 (3.1)    & 1.54 (2.6)   & 2.03 (25.4)            & 5.46 (0.2)           & 2.14 (1.1)          & 3.1 (11.3)          \\
      30                   & 50k                  & 0.23 (17.7)  & 0.23 (10.4)  & 0.33 (52.9)            & 0.34 (1.1)           & 0.40 (8.8)           & 0.25 (28.8)         & 0.95 (13.0)  & 0.96 (12.9)  & 1.06 (48.9)            & 2.90 (1.1)            & 1.66 (9.3)          & 1.66 (12.0)         \\
      30                   & 250k                 & 0.14 (166.4) & 0.15 (65.5)  & 0.24 (513.3)           & 0.20 (11.8)           & 0.27 (91.6)         & 0.18 (37.3)         & 0.58 (119.9) & 0.58 (83.0)    & 0.62 (503)             & 1.46 (10.8)          & 1.37 (104.4)        & 0.81 (15.9)         \\
      \midrule
      \multicolumn{14}{c}{Single Index}\\
      \midrule
      500                  & 300                  & 6.18 (3.8)   & 6.22 (3.5)   & 8.11 (3.6)             & 8.19 (0.1)           & 8.40 (0.1)           & 20.43 (2.6)         & 19.76 (3.7)  & 14.7 (3.3)   & 26.82 (3.9)            & 40.52 (0.1)          & 12.23 (0.1)         & 34.81 (1.4)         \\
      100                  & 1k                   & 3.95 (1.8)   & 4.00 (1.7)      & 5.25 (7.1)             & 5.98 (0.1)           & 6.00 (0.2)             & 10.51 (26.9)        & 12.96 (1.6)  & 9.48 (1.5)   & 18.33 (7.3)            & 35.46 (0.1)          & 11.05 (0.2)         & 38.59 (11)          \\
      1k                   & 500                  & 4.84 (8.7)   & 4.75 (8.3)   & 7.17 (10.8)            & 7.59 (0.3)           & 7.89 (0.3)          & 14.07 (40.6)        & 14.65 (8.3)  & 10.14 (8.3)  & 23.81 (11.1)           & 38.53 (0.3)          & 11.6 (0.2)          & 39.23 (16.4)        \\
      1k                   & 1k                   & 4.08 (12.0)    & 4.12 (12.1)  & 5.93 (21.1)            & 6.78 (0.6)           & 7.86 (0.4)          & 12.01 (27.9)        & 12.96 (10.5) & 9.98 (12.8)  & 20.41 (21.7)           & 35.02 (0.6)          & 10.54 (0.4)         & 50 (11.6)           \\
      30                   & 10k                  & 2.38 (5.8)   & 2.30 (4.4)    & 2.79 (25.5)            & 3.36 (0.3)           & 3.73 (1.1)          & 2.66 (27.5)         & 6.41 (3.2)   & 6.09 (2.6)   & 8.25 (25.4)            & 20.61 (0.2)          & 8.06 (1.1)          & 8.49 (11.2)         \\
      30                   & 50k                  & 1.68 (41.5)  & 1.66 (33.3)  & 2.22 (49.5)            & 2.35 (1.2)           & 3.08 (7.9)          & 1.96 (29.1)         & 4.45 (15.2)  & 4.41 (13.0)    & 5.51 (49.9)            & 11.71 (1.1)          & 6.62 (9.3)          & 6.45 (12.1)         \\
      30                   & 250k                 & 1.23 (518.4) & 1.22 (326.6) & 1.99 (489.2)           & 1.81 (11.0)            & 2.59 (73.0)           & 1.65 (37.8)         & 3.11 (176.3) & 2.98 (116.9) & 4.58 (508.5)           & 6.37 (9.7)           & 5.62 (94.4)         & 4.61 (16.0)           \\
      \midrule
      \multicolumn{14}{c}{Trig+Poly}\\
      \midrule
      500                  & 300                  & 4.39 (3.9)   & 5.03 (3.3)   & 5.20 (3.9)              & 5.74 (0.1)           & 5.41 (0.1)          & 7.08 (2.6)          & 12.02 (3.6)  & 7.84 (3.3)   & 15.34 (3.9)            & 25.23 (0.1)          & 7.56 (0.1)          & 9.55 (1.3)          \\
      100                  & 1k                   & 3.00 (1.8)      & 3.14 (1.6)   & 3.82 (7.3)             & 4.39 (0.2)           & 4.70 (0.2)           & 9.18 (27)           & 8.71 (1.6)   & 6.23 (1.5)   & 12.87 (7.3)            & 25.29 (0.1)          & 7.59 (0.2)          & 37.80 (11.1)         \\
      1k                   & 500                  & 3.26 (9.0)     & 3.13 (8.5)   & 4.18 (11.3)            & 4.54 (0.3)           & 5.01 (0.2)          & 6.46 (40.5)         & 9.81 (8.4)   & 6.32 (8.3)   & 14.6 (11.0)              & 23.34 (0.3)          & 6.22 (0.2)          & 26.66 (16.5)        \\
      1k                   & 1k                   & 2.90 (12.2)   & 2.85 (13.3)  & 4.04 (21.8)            & 4.74 (0.6)           & 5.21 (0.5)          & 7.22 (27.8)         & 9.27 (10.7)  & 6.43 (12.1)  & 13.29 (21.3)           & 23.37 (0.6)          & 6.36 (0.5)          & 33.26 (11.6)        \\
      30                   & 10k                  & 1.25 (5.0)   & 1.52 (3.9)   & 2.42 (26)              & 2.77 (0.2)           & 3.26 (1.1)          & 3.73 (27.3)         & 4.57 (3.1)   & 4.53 (2.6)   & 5.61 (25.4)            & 13.62 (0.2)          & 5.89 (1.2)          & 7.72 (11.3)         \\
      30                   & 50k                  & 0.74 (27.4)  & 0.75 (15.0)    & 2.30 (49.9)             & 1.96 (1.1)           & 2.83 (9.2)          & 3.26 (29.0)         & 3.06 (14.3)  & 3.07 (12.8)  & 4.02 (48.3)            & 7.44 (1.1)           & 4.92 (9.5)          & 5.56 (12.3)         \\
      30                   & 250k                 & 0.48 (273.1) & 0.50 (134.3) & 1.72 (504.3)           & 1.18 (10.8)          & 2.45 (93.1)         & 2.20 (37.6)          & 2.28 (146.6) & 2.01 (104.4) & 3.35 (507.1)           & 4.22 (9.7)           & 4.24 (103.6)        & 4.12 (16.2)         \\
      \bottomrule
    \end{tabular}%
  }
  \caption{Root mean squared error (RMSE) and running time in seconds (in parenthesis). Column $p$ and $n$ are number of regressors and observations respectively. Regressors $X$ independent and the noise $\epsilon$ is Gaussian. The left panel is for noise level $\kappa = 1$ and the right panel is for higher noise level $\kappa = 10$. }\label{table:ind}
\end{table}%

\newpage

\begin{table}[!htbp]
  \centering
  \resizebox{\columnwidth}{!}{%
    \begin{tabular}{rc|cccccc|cccccc}
      \toprule
      \multicolumn{3}{c}{} & \multicolumn{5}{c}{$\kappa = 1$} & \multicolumn{5}{c}{$\kappa = 10$}                                                                                                                                          \\
      \midrule
      \multirow{2}{*}{$p$} & \multirow{2}{*}{$n$} & XBART        & XBART        & \multirow{2}{*}{XGBCV} & \multirow{2}{*}{XGB} & \multirow{2}{*}{RF} & \multirow{2}{*}{NN} & XBART        & XBART        & \multirow{2}{*}{XGBCV} & \multirow{2}{*}{XGB} & \multirow{2}{*}{RF} & \multirow{2}{*}{NN} \\
                           &                      & fixed $\tau$ & sampling $\tau$  &                        &                      &                     &                     & fixed $\tau$ & sampling $\tau$  &                        &                      &                     &                     \\
      \midrule
      \multicolumn{14}{c}{Linear}\\
      \midrule
      500                  & 300                  & 40.11 (3.9)  & 40.32 (3.5)  & 40.22 (3.9)            & 44.36 (0.1)          & 38.62 (0.1)         & 28.90 (2.7)          & 94.38 (3.7)  & 65.83 (3.4)  & 122.51 (3.8)           & 190.81 (0.1)         & 63.50 (0.1)          & 41.00 (1.3)            \\
      100                  & 1k                   & 15.33 (2.0)    & 15.37 (1.7)  & 12.57 (7.3)            & 13.73 (0.2)          & 13.23 (0.2)         & 19.19 (26.9)        & 33.48 (1.6)  & 21.15 (1.6)  & 43.43 (7.1)            & 77.66 (0.2)          & 21.41 (0.2)         & 101.98 (11.1)       \\
      1k                   & 500                  & 61.64 (8.5)  & 60.73 (9.5)  & 61.16 (10.8)           & 65.83 (0.3)          & 58.92 (0.2)         & 45.1 (40.7)         & 121.94 (8.5) & 74.34 (9.4)  & 174.87 (10.9)          & 272.60 (0.3)          & 78.78 (0.3)         & 290.26 (16.6)       \\
      1k                   & 1k                   & 58.97 (12.2) & 58.70 (13.8)  & 58.44 (21.2)           & 63.90 (0.6)           & 56.76 (0.4)         & 45.48 (27.9)        & 97.75 (10.5) & 71.13 (13.4) & 153.25 (21.5)          & 283.11 (0.6)         & 72.73 (0.4)         & 369.02 (11.4)       \\
      30                   & 10k                  & 2.52 (7.0)   & 2.50 (5.6)    & 3.51 (27.8)            & 4.01 (0.3)           & 3.60 (1.0)           & 1.89 (27.5)         & 7.22 (3.3)   & 5.73 (2.7)   & 8.57 (25.9)            & 23.72 (0.2)          & 8.59 (1.1)          & 9.64 (11.6)         \\
      30                   & 50k                  & 1.73 (50.9)  & 1.74 (40.2)  & 2.88 (52.1)            & 3.14 (1.1)           & 3.13 (7.0)            & 0.86 (28.8)         & 4.80 (17.5)   & 4.26 (15.8)  & 5.30 (49.7)             & 13.06 (1.1)          & 6.46 (8.9)          & 4.44 (12.0)           \\
      30                   & 250k                 & 1.13 (614.3) & 1.12 (313.0)   & 2.59 (495.6)           & 2.64 (7.9)           & 2.78 (57.9)         & 0.42 (38.3)         & 3.56 (222.7) & 3.27 (133.3) & 4.43 (498.1)           & 7.22 (16.2)          & 5.51 (82.8)         & 2.18 (15.6)         \\
      \midrule
      \multicolumn{14}{c}{Max}\\
      \midrule
      500                  & 300                  & 1.53 (3.9)   & 1.52 (3.4)   & 1.35 (3.9)             & 1.46 (0.1)           & 2.00 (0.1)             & 2.56 (2.7)          & 5.12 (3.8)   & 3.06 (3.3)   & 6.31 (3.8)             & 9.77 (0.1)           & 2.94 (0.1)          & 7.10 (1.3)           \\
      100                  & 1k                   & 0.94 (1.8)   & 0.92 (1.5)   & 0.87 (7.3)             & 1.22 (0.2)           & 1.46 (0.2)          & 2.82 (27.0)           & 3.40 (1.6)    & 2.60 (1.5)    & 5.32 (7.3)             & 9.70 (0.2)            & 3.28 (0.2)          & 15.59 (11.0)          \\
      1k                   & 500                  & 1.23 (9.0)   & 1.27 (9.8)   & 1.20 (11.1)             & 1.45 (0.3)           & 2.08 (0.3)          & 2.27 (40.8)         & 3.58 (8.4)   & 2.71 (9.6)   & 5.94 (10.7)            & 9.06 (0.3)           & 2.78 (0.2)          & 9.88 (16.7)         \\
      1k                   & 1k                   & 0.98 (11.3)  & 0.94 (13.3)  & 0.98 (21.9)            & 1.24 (0.6)           & 2.05 (0.5)          & 2.46 (28.3)         & 3.65 (10.3)  & 2.71 (14.1)  & 5.36 (21.2)            & 9.25 (0.6)           & 2.71 (0.5)          & 13.27 (11.7)        \\
      30                   & 10k                  & 0.43 (4.0)   & 0.44 (2.2)   & 0.45 (27.1)            & 0.68 (0.2)           & 0.73 (1.1)          & 0.44 (27.7)         & 1.73 (3.1)   & 1.62 (2.6)   & 2.22 (25.6)            & 5.86 (0.2)           & 2.30 (1.1)           & 2.98 (11.4)         \\
      30                   & 50k                  & 0.24 (18.1)  & 0.26 (9.4)   & 0.37 (51.5)            & 0.38 (1.1)           & 0.50 (8.2)           & 0.23 (29.1)         & 1.05 (13.1)  & 1.08 (11.9)  & 1.18 (49.5)            & 3.11 (1.1)           & 1.80 (9.1)           & 1.75 (11.9)         \\
      30                   & 250k                 & 0.16 (171.4) & 0.17 (66.8)  & 0.27 (491.2)           & 0.21 (7.7)           & 0.36 (86.4)         & 0.22 (36.8)         & 0.65 (125.4) & 0.63 (88.9)  & 0.67 (534.6)           & 1.61 (15.1)          & 1.49 (97.4)         & 0.90 (21.8)          \\
      \midrule
      \multicolumn{14}{c}{Single Index}\\
      \midrule
      500                  & 300                  & 5.41 (3.8)   & 5.35 (3.3)   & 7.86 (3.7)             & 7.94 (0.1)           & 8.22 (0.1)          & 19.49 (2.6)         & 19.27 (3.7)  & 12.72 (3.3)  & 26.03 (3.8)            & 40.03 (0.1)          & 11.76 (0.1)         & 32.54 (1.3)         \\
      100                  & 1k                   & 4.08 (1.8)   & 4.08 (1.6)   & 5.56 (7.1)             & 6.38 (0.2)           & 6.32 (0.2)          & 10.04 (27.2)        & 12.72 (1.6)  & 9.43 (1.5)   & 20.03 (7.2)            & 38.18 (0.1)          & 11.68 (0.2)         & 34.88 (11.1)        \\
      1k                   & 500                  & 4.80 (8.7)    & 4.78 (8.2)   & 6.57 (10.4)            & 7.49 (0.3)           & 7.92 (0.2)          & 13.63 (40.6)        & 14.94 (8.4)  & 10.91 (8.3)  & 24.31 (10.9)           & 37.17 (0.3)          & 10.23 (0.2)         & 41.36 (16.4)        \\
      1k                   & 1k                   & 3.97 (12.1)  & 3.92 (12.6)  & 5.92 (21.1)            & 6.48 (0.6)           & 7.51 (0.4)          & 12.29 (28.0)        & 12.43 (10.9) & 9.91 (13.8)  & 20.45 (21.4)           & 35.57 (0.6)          & 9.99 (0.5)          & 50.13 (11.5)        \\
      30                   & 10k                  & 2.53 (5.4)   & 2.58 (3.5)   & 3.10 (26.0)             & 3.78 (0.2)           & 3.73 (1.0)          & 3.17 (27.4)         & 7.30 (3.2)    & 6.88 (2.6)   & 9.46 (25.7)            & 25.15 (0.2)          & 9.73 (1.1)          & 9.59 (11.3)         \\
      30                   & 50k                  & 1.80 (38.9)   & 1.80 (29.3)   & 2.41 (50.4)            & 2.64 (1.2)           & 3.08 (7.6)          & 2.10 (29.7)          & 4.73 (15.1)  & 4.62 (13.4)  & 5.71 (49.8)            & 13.86 (1.1)          & 7.44 (9.3)          & 6.44 (12.0)           \\
      30                   & 250k                 & 1.31 (532.7) & 1.34 (331.2) & 2.16 (522.6)           & 2.08 (7.9)           & 2.60 (74.8)          & 1.75 (38.7)         & 3.43 (187.7) & 3.36 (118.7) & 4.48 (521.6)           & 7.50 (11.3)           & 6.43 (109.6)        & 4.61 (15.6)         \\
      \midrule
      \multicolumn{14}{c}{Trig+Poly}\\
      \midrule
      500                  & 300                  & 3.72 (3.9)   & 4.13 (3.4)   & 4.72 (3.8)             & 5.16 (0.1)           & 4.96 (0.1)          & 6.71 (2.6)          & 11.26 (3.6)  & 7.43 (3.3)   & 17.45 (3.9)            & 24.04 (0.1)          & 7.85 (0.1)          & 10.06 (1.4)         \\
      100                  & 1k                   & 2.89 (1.8)   & 2.78 (1.5)   & 3.68 (7.3)             & 4.43 (0.2)           & 4.64 (0.2)          & 9.44 (27.1)         & 8.12 (1.6)   & 6.20 (1.5)    & 12.42 (7.1)            & 24.91 (0.1)          & 7.80 (0.2)           & 36.26 (11.0)          \\
      1k                   & 500                  & 3.58 (8.9)   & 3.53 (9.9)   & 4.72 (11.0)            & 5.00 (0.3)              & 5.14 (0.2)          & 6.40 (40.5)          & 10.51 (8.3)  & 7.71 (8.3)   & 14.19 (10.8)           & 23.39 (0.3)          & 6.99 (0.2)          & 25.79 (16.4)        \\
      1k                   & 1k                   & 2.92 (12.1)  & 3.02 (12.3)  & 4.28 (21.8)            & 4.80 (0.6)            & 5.26 (0.5)          & 7.24 (27.8)         & 8.04 (10.4)  & 6.48 (13.2)  & 12.73 (21.4)           & 23.55 (0.6)          & 6.72 (0.5)          & 32.92 (11.5)        \\
      30                   & 10k                  & 1.24 (5.1)   & 1.21 (4.1)   & 1.99 (26.0)            & 2.27 (0.2)           & 2.96 (1.1)          & 3.59 (27.4)         & 4.80 (3.2)    & 4.88 (2.7)   & 5.88 (25.4)            & 14.02 (0.2)          & 6.24 (1.1)          & 7.70 (11.3)          \\
      30                   & 50k                  & 0.75 (28.4)  & 0.80 (17.3)   & 1.36 (49.8)            & 1.42 (1.1)           & 2.23 (8.5)          & 2.66 (28.7)         & 3.01 (14.4)  & 3.05 (13.2)  & 3.83 (49.1)            & 7.65 (1.1)           & 4.87 (9.2)          & 5.43 (12)           \\
      30                   & 250k                 & 0.50 (336.4)  & 0.51 (152.1) & 1.26 (498)             & 1.05 (7.6)           & 1.71 (81.5)         & 0.96 (36.9)         & 2.00 (167.1)    & 1.88 (94.3)  & 2.71 (496.9)           & 4.05 (7.7)           & 3.99 (112.7)        & 4.27 (15.6)         \\
      \bottomrule
    \end{tabular}%
  }
  \caption{Root mean squared error (RMSE) and running time in seconds (in parenthesis). Column $p$ and $n$ are number of regressors and observations respectively. Regressors $X$ are correlated with factor structure and the noise $\epsilon$ is Gaussian. The left panel is for noise level $\kappa = 1$ and the right panel is for higher noise level $\kappa = 10$. }\label{table:factor}
\end{table}%

\newpage

\begin{table}[!htbp]
  \centering
  \resizebox{\columnwidth}{!}{%
    \begin{tabular}{rc|cccccc|cccccc}
      \toprule
      \multicolumn{3}{c}{} & \multicolumn{5}{c}{$\kappa = 1$} & \multicolumn{5}{c}{$\kappa = 10$}                                                                                                                                          \\
      \midrule
      \multirow{2}{*}{$p$} & \multirow{2}{*}{$n$} & XBART        & XBART        & \multirow{2}{*}{XGBCV} & \multirow{2}{*}{XGB} & \multirow{2}{*}{RF} & \multirow{2}{*}{NN} & XBART        & XBART        & \multirow{2}{*}{XGBCV} & \multirow{2}{*}{XGB} & \multirow{2}{*}{RF} & \multirow{2}{*}{NN} \\
                           &                      & fixed $\tau$ & sampling $\tau$  &                        &                      &                     &                     & fixed $\tau$ & sampling $\tau$  &                        &                      &                     &                     \\
      \midrule
      \multicolumn{14}{c}{Linear}\\
      \midrule
      500                  & 300                  & 25.42 (3.9)  & 25.13 (3.3)  & 25.70 (3.8)             & 28.31 (0.1)          & 25.03 (0.1)         & 24.24 (2.7)         & 48.86 (3.7)  & 34.46 (3.5)  & 73.34 (3.9)            & 144.63 (0.1)         & 40.43 (0.2)         & 27.04 (1.3)         \\
      100                  & 1k                   & 9.88 (2.1)   & 9.87 (1.7)   & 9.39 (7.7)             & 9.87 (0.2)           & 9.76 (0.2)          & 12.52 (27.0)          & 18.08 (1.7)  & 12.55 (1.6)  & 25.73 (7.4)            & 48.01 (0.1)          & 17.42 (0.2)         & 58.65 (11.0)          \\
      1k                   & 500                  & 37.09 (8.6)  & 36.31 (8.4)  & 37.95 (11.3)           & 39.16 (0.3)          & 35.71 (0.2)         & 33.95 (40.9)        & 41.29 (8.6)  & 37.49 (8.5)  & 46.34 (11.6)           & 62.18 (0.3)          & 36.9 (0.2)          & 56.26 (16.6)        \\
      1k                   & 1k                   & 36.01 (11.8) & 35.53 (14.8) & 36.75 (22.1)           & 38.39 (0.6)          & 35.14 (0.5)         & 31.46 (28.2)        & 58.62 (10.8) & 42.3 (13.3)  & 87.26 (22)             & 153.16 (0.6)         & 45.97 (0.5)         & 199.86 (11.5)       \\
      30                   & 10k                  & 2.13 (8.2)   & 2.14 (5.7)   & 3.17 (26.9)            & 3.29 (0.3)           & 3.68 (1.0)            & 1.37 (27.2)         & 5.34 (3.4)   & 4.81 (2.8)   & 7.11 (24.5)            & 16.37 (0.2)          & 6.73 (1.2)          & 7.08 (11.3)         \\
      30                   & 50k                  & 1.44 (59.7)  & 1.49 (44.1)  & 2.93 (52.4)            & 2.75 (1.2)           & 3.31 (7.1)          & 0.60 (28.8)          & 3.82 (18.5)  & 3.60 (16.9)   & 5.08 (48.8)            & 10.72 (1.1)          & 5.12 (9.6)          & 3.56 (12.0)         \\
      30                   & 250k                 & 0.91 (633.1) & 1.04 (406.2) & 2.8 (483.3)            & 2.52 (8.2)           & 3.02 (59.0)           & 0.34 (37.6)         & 3.18 (221.8) & 3.90 (170.3)  & 5.01 (469.1)           & 8.29 (9.1)           & 4.31 (90.1)         & 1.89 (15.8)         \\
      \midrule
      \multicolumn{14}{c}{Max}          \\
      \midrule
      500                  & 300                  & 1.48 (3.9)   & 1.53 (3.4)   & 1.32 (3.8)             & 1.54 (0.1)           & 1.94 (0.1)          & 2.57 (2.6)          & 4.85 (3.7)   & 3.06 (3.5)   & 6.77 (4.0)               & 10.84 (0.1)          & 3.39 (0.1)          & 5.48 (1.3)          \\
      100                  & 1k                   & 0.81 (1.8)   & 0.79 (1.5)   & 0.82 (7.6)             & 1.14 (0.1)           & 1.35 (0.2)          & 2.68 (26.7)         & 3.35 (1.6)   & 2.46 (1.5)   & 5.07 (7.3)             & 10.64 (0.1)          & 3.79 (0.2)          & 15.17 (11)          \\
      1k                   & 500                  & 1.27 (9.1)   & 1.24 (8.6)   & 1.17 (11.5)            & 1.46 (0.3)           & 1.98 (0.3)          & 2.25 (40.7)         & 1.97 (8.6)   & 1.79 (8.4)   & 2.49 (11.2)            & 4.10 (0.3)            & 2.14 (0.3)          & 3.71 (16.6)         \\
      1k                   & 1k                   & 0.80 (12.1)   & 0.81 (13.4)  & 0.86 (22.9)            & 1.24 (0.6)           & 1.89 (0.5)          & 2.31 (27.8)         & 3.28 (10.6)  & 2.32 (13.5)  & 5.30 (22.2)             & 9.77 (0.6)           & 2.90 (0.5)           & 12.29 (11.5)        \\
      30                   & 10k                  & 0.40 (4.0)      & 0.44 (2.1)   & 0.46 (25.6)            & 0.66 (0.2)           & 0.59 (1.2)          & 0.38 (27.1)         & 1.66 (3.1)   & 1.55 (2.5)   & 2.15 (24.3)            & 5.80 (0.2)            & 2.47 (1.2)          & 2.84 (11.4)         \\
      30                   & 50k                  & 0.24 (18.5)  & 0.28 (9.4)   & 0.34 (50.1)            & 0.41 (1.1)           & 0.38 (9.2)          & 0.24 (28.7)         & 1.00 (13.7)     & 1.40 (11.7)   & 1.20 (47.7)             & 3.58 (1.1)           & 1.88 (10.1)         & 1.56 (11.9)         \\
      30                   & 250k                 & 0.15 (161.9) & 0.18 (69.1)  & 0.26 (448.2)           & 0.23 (7.8)           & 0.25 (91.8)         & 0.19 (37.9)         & 0.58 (119.7) & 0.59 (87.5)  & 0.70 (464.7)            & 1.75 (7.8)           & 1.47 (108.7)        & 0.82 (15.8)         \\
      \midrule
      \multicolumn{14}{c}{Single Index}          \\
      \midrule
      500                  & 300                  & 5.46 (3.8)   & 5.41 (3.3)   & 7.61 (3.8)             & 8.41 (0.1)           & 8.02 (0.1)          & 18.94 (2.8)         & 20.39 (3.7)  & 14.07 (3.4)  & 29.98 (4.1)            & 35.83 (0.1)          & 13.96 (0.1)         & 33.75 (1.3)         \\
      100                  & 1k                   & 4.34 (1.8)   & 4.17 (1.6)   & 5.55 (6.9)             & 6.32 (0.1)           & 6.28 (0.2)          & 10.31 (27.0)          & 13.60 (1.7)   & 9.66 (1.7)   & 18.33 (7.2)            & 36.31 (0.2)          & 13.88 (0.2)         & 35.01 (11.1)        \\
      1k                   & 500                  & 4.77 (8.8)   & 4.74 (9.6)   & 6.91 (10.5)            & 7.68 (0.3)           & 8.13 (0.3)          & 13.47 (40.9)        & 6.94 (8.7)   & 6.84 (8.1)   & 10.93 (11.5)           & 13.47 (0.3)          & 8.68 (0.3)          & 19.72 (16.6)        \\
      1k                   & 1k                   & 4.01 (11.9)  & 3.87 (14.3)  & 6.05 (21.1)            & 6.67 (0.6)           & 7.69 (0.5)          & 12.58 (28.0)          & 13.42 (10.6) & 9.74 (11.9)  & 19.77 (22.1)           & 34.75 (0.6)          & 11.13 (0.5)         & 45.43 (11.5)        \\
      30                   & 10k                  & 2.37 (5.9)   & 2.46 (4.9)   & 2.81 (24.9)            & 3.32 (0.2)           & 3.65 (1.1)          & 2.69 (27.4)         & 6.40 (3.3)    & 7.06 (2.4)   & 8.78 (24.9)            & 23.68 (0.2)          & 9.74 (1.2)          & 9.28 (11.4)         \\
      30                   & 50k                  & 1.65 (42.5)  & 2.46 (31.3)  & 2.26 (49.0)            & 2.47 (1.1)           & 3.05 (8.2)          & 1.93 (28.9)         & 4.58 (15.7)  & 12.21 (11.6) & 6.51 (48.3)            & 14.23 (1.1)          & 7.29 (9.8)          & 6.45 (11.9)         \\
      30                   & 250k                 & 1.22 (524.5) & 1.28 (352.7) & 2.01 (455.1)           & 1.89 (7.9)           & 2.53 (83.2)         & 1.68 (37.2)         & 3.17 (179.0)   & 3.06 (117.5) & 4.75 (468.5)           & 7.85 (8.6)           & 6.00 (101.6)           & 4.66 (16.0)           \\
      \midrule
      \multicolumn{14}{c}{Trig+Poly}          \\
      \midrule
      500                  & 300                  & 4.22 (3.9)   & 4.42 (3.3)   & 4.89 (3.7)             & 5.20 (0.1)            & 5.24 (0.1)          & 6.41 (2.6)          & 11.48 (3.7)  & 7.40 (3.4)    & 14.05 (4.0)              & 31.32 (0.1)          & 7.61 (0.1)          & 7.19 (1.4)          \\
      100                  & 1k                   & 2.80 (1.8)    & 2.90 (1.5)    & 3.68 (7.3)             & 4.27 (0.2)           & 4.60 (0.2)           & 9.14 (26.9)         & 8.36 (1.6)   & 6.21 (1.5)   & 12.82 (7.3)            & 29.21 (0.1)          & 9.64 (0.2)          & 33.32 (11.0)          \\
      1k                   & 500                  & 3.89 (9.2)   & 4.03 (8.7)   & 4.89 (11.5)            & 5.75 (0.3)           & 5.34 (0.3)          & 6.74 (40.5)         & 5.98 (8.7)   & 5.64 (9.5)   & 7.01 (11.6)            & 12.71 (0.3)          & 5.71 (0.3)          & 9.98 (16.6)         \\
      1k                   & 1k                   & 2.86 (11.9)  & 2.93 (12.5)  & 4.11 (21.8)            & 4.88 (0.6)           & 5.24 (0.5)          & 7.02 (28.2)         & 8.81 (10.5)  & 6.79 (11.6)  & 13.44 (22.3)           & 27.91 (0.6)          & 7.91 (0.5)          & 33.00 (11.5)           \\
      30                   & 10k                  & 1.22 (5.0)   & 1.27 (2.9)   & 2.46 (25.7)            & 2.39 (0.2)           & 3.28 (1.2)          & 3.95 (27.4)         & 4.66 (3.2)   & 4.78 (2.5)   & 6.02 (24.6)            & 13.07 (0.2)          & 6.65 (1.2)          & 7.94 (11.3)         \\
      30                   & 50k                  & 0.77 (28.4)  & 1.30 (14.5)   & 2.23 (48.5)            & 2.02 (1.1)           & 2.86 (9.5)          & 3.62 (28.9)         & 3.36 (13.8)  & 3.79 (11.0)    & 4.23 (47.3)            & 8.33 (1.1)           & 5.29 (10.1)         & 5.43 (12.2)         \\
      30                   & 250k                 & 1.77 (247.8) & 1.53 (121.9) & 1.87 (481.3)           & 1.82 (7.7)           & 2.48 (95.8)         & 2.36 (37.3)         & 2.37 (136.8) & 2.89 (100.0)   & 3.59 (466.5)           & 5.86 (7.9)           & 4.51 (110.8)        & 4.21 (15.8)         \\
      \bottomrule
    \end{tabular}%
  }
  \caption{Root mean squared error (RMSE) and running time in seconds (in parenthesis). Column $p$ and $n$ are number of regressors and observations respectively. Regressors $X$ are independent and the noise $\epsilon$ is $t$ distributed with degree of freedom 3. The left panel is for noise level $\kappa = 1$ and the right panel is for higher noise level $\kappa = 10$. }\label{table:t}
\end{table}

\end{document}